\documentclass[11pt]{article}

\usepackage{latexsym,mathrsfs}
\usepackage{amsmath,amssymb}
\usepackage{amsthm,enumerate,verbatim}
\usepackage{amsfonts}
\usepackage{graphicx}
\usepackage{algorithm}
\usepackage{algorithmic}
\usepackage{url}

\renewcommand{\algorithmicrequire}{\textbf{Input:}}
\renewcommand{\algorithmicensure}{\textbf{Output:}}

\newtheorem{lemma}{Lemma}

\newtheorem{theorem}{Theorem}
\newtheorem{corollary}{Corollary}

\newtheorem*{definition*}{Definition}

\newtheorem{assumption}{Assumption}
\newtheorem{remark}{Remark}

\DeclareMathOperator{\argmax}{argmax} 
\DeclareMathOperator{\argmin}{argmin} 
 
\DeclareMathOperator{\rank}{rank}

\title{Successive Nonnegative Projection Algorithm \\ for Robust Nonnegative Blind Source Separation}
\date{}

\author{Nicolas Gillis \\ 
Department of Mathematics and Operational Research \\ 
Facult\'e Polytechnique, Universit\'e de Mons \\ 
Rue de Houdain 9, 7000 Mons, Belgium\\
 nicolas.gillis@umons.ac.be  
}

\begin{document}

\maketitle

\begin{abstract} 
In this paper, we propose a new fast and robust recursive algorithm for near-separable nonnegative matrix factorization, a particular nonnegative blind source separation problem. This algorithm, which we refer to as the successive nonnegative projection algorithm (SNPA), is closely related to the popular successive projection algorithm (SPA), but takes advantage of the nonnegativity constraint in the decomposition. We prove that SNPA is more robust than SPA and can be applied to a broader class of nonnegative matrices. This is illustrated on some synthetic data sets, and on a real-world hyperspectral image.  
\end{abstract} 


\section{Introduction}

Nonnegative matrix factorization (NMF)  has become a widely used tool for analysis of high-dimensional data. 
NMF decomposes approximately a nonnegative input data matrix $M \in \mathbb{R}^{m \times n}_+$ into the product of two nonnegative matrices $W \in \mathbb{R}^{m \times r}_+$ and $H \in \mathbb{R}^{r \times n}_+$ so that $M \approx WH$. 
Although NMF is NP-hard in general \cite{V09} and ill-posed (see \cite{G12} and the references therein), it has been used in many different areas such as image processing~\cite{LS99},  document classification~\cite{SBPP06}, hyperspectral unmixing~\cite{PPP06}, community detection~\cite{WLW11}, and computational biology~\cite{D08}. 
Recently, Arora et al.~\cite{AGKM11} introduced a subclass of nonnegative matrices, referred to as separable, for which NMF can be solved efficiently (that is, in polynomial time), even in the presence of noise. This subclass of NMF problems are referred to as near-separable NMF, and has been shown to be useful in several applications such as 
document classification \cite{AGM12, Ar13, KSK12, DRIS13},  
blind source separation \cite{CMCW08}, 
video summarization and  image
classification \cite{ESV12}, and 
hyperspectral unmixing (see Section~\ref{nsnmf} below). 


\subsection{Near-Separable NMF} \label{nsnmf}

A matrix $M$ is $r$-separable if there exists an index set $\mathcal{K}$ of cardinality $r$ and a nonnegative matrix $H \in \mathbb{R}^{r \times n}_+$ with \mbox{$M = M(:,\mathcal{K}) H$}. Equivalently, $M$ is $r$-separable if \vspace{-0.2cm}
\[
M = W \, [I_r, \, H'] \, \Pi, 
\]
where $I_r$ is the $r$-by-$r$ identity matrix, $H'$ is a nonnegative matrix and $\Pi$ is a permutation. Given a separable matrix, the goal is to identify the $r$ columns of $M$ allowing to reconstruct it perfectly, that is, to identify the columns of $M$ corresponding the columns of $W$. 
In the presence of noise, the problem is referred to as near-separable NMF and can be stated as follows.  \vspace{0.2cm} 

\noindent \textbf{Near-Separable NMF:}
Given the noisy $r$-separable matrix $\tilde{M} = WH + N \in \mathbb{R}^{m \times n}$ where 
$N$ is the noise, $W \in \mathbb{R}^{m \times r}_+$, $H = [I_r, H']\Pi$ with $H' \geq 0$ and $\Pi$ is a permutation, recover approximately the columns of $W$ among the columns of~$\tilde{M}$. \vspace{0.2cm}

\noindent An important application of near-separable NMF is blind hyperspectral unmixing in the presence of pure pixels \cite{GV12, Ma14}: A hyperspectral image is a set of images taken at different wavelengths. It can be associated with a nonnegative matrix $M \in \mathbb{R}^{m \times n}_+$ where $m$ is the number of wavelengths and $n$ the number of pixels. Each column of $M$ is equal to the spectral signature of a given pixel, that is, $M(i,j)$ is the fraction of incident light reflected by the $j$th pixel at the $i$th wavelength. Under the linear mixing model, the spectral signature of a pixel is equal to a linear combination of the spectral signatures of the constitutive materials present in the image, referred to as endmembers. 
The weights in that linear combination are nonnegative and sum to one, and correspond to the abundances of the endmembers in that pixel. 
If for each endmember, there exists a pixel in the image containing only that endmember, then the pure-pixel assumption is satisfied. This assumption is equivalent to the separability assumption: each column of $W$ is the spectral signature of an endmember and is equal to a column of $M$ corresponding to a pure pixel; 
see the survey \cite{BP12} for more details. \\

 Several provably robust algorithms have been proposed to solve the near-separable NMF problem using, e.g.,  geometric constructions~\cite{AGKM11,Ar13}, linear programming~\cite{EMO12, BRRT12, G13, GL13}, or semidefinite programming~\cite{M13, GV13}. 
In the next section, we briefly describe the successive projection algorithm (SPA) which is closely related to the algorithm we propose in this paper.

\subsection{Successive Projection Algorithm} The successive projection algorithm is a simple but fast and robust recursive algorithm for solving near-separable NMF; 
see Algorithm~\ref{spa}. At each step of the algorithm, the column of the input matrix $\tilde{M}$ with maximum $\ell_2$ norm  is selected, and then  $\tilde{M}$ is updated by projecting each column onto the orthogonal complement of the columns selected so far. It was first introduced in \cite{MC01}, and later proved to be robust~in~\cite{GV12}. \vspace{-0.2cm}
\renewcommand{\thealgorithm}{SPA}
\algsetup{indent=2em}
\begin{algorithm}[ht!]
\caption{Successive Projection Algorithm \cite{MC01, GV12} \label{spa}}
\begin{algorithmic}[1] 
\REQUIRE Near-separable matrix $\tilde{M} = WH + N \in \mathbb{R}^{m \times n}$ satisfying Ass.~\ref{asssep} where $W$ has full column rank,  the number $r$ of columns to be extracted. 
\ENSURE Set of indices $\mathcal{K}$ such that $M(:,\mathcal{K}) \approx W$ (up to permutation). 
    \medskip  

\STATE Let $R = \tilde{M}$, $\mathcal{K} = \{\}$, $k = 1$. 
\WHILE {$R \neq 0$ and $k \leq r$}   
\STATE $p = \argmax_j ||R_{:j}||_2$. $^\dagger$ 
\STATE $R = \left(I-\frac{{R_{:p}} R_{:p}^T}{||{R_{:p}}||_2^2}\right)R$. \vspace{0.1cm} 
\STATE $\mathcal{K} = \mathcal{K} \cup \{p\}$. 
\STATE $k = k + 1$. 
\ENDWHILE
\end{algorithmic}  
$^\dagger$ In case of a tie, the index $j$ whose corresponding column of the original matrix $\tilde{M}$ maximizes $f$ is selected. In case of another tie, one of these columns is picked randomly. 
\end{algorithm}  
\begin{theorem}[\cite{GV12}, Th.~3] \label{th1} 
Let $\tilde{M} = WH + N$ be a near-separable matrix  (see Assumption~\ref{asssep}) where $W$ has full column rank and $\max_i ||N(:,i)||_2 \leq \epsilon$. 
If $\epsilon \leq \mathcal{O} \left( \,  \frac{  \sigma_{\min}(W)  }{\sqrt{r} \kappa^2(W)} \right)$, then \ref{spa} identifies the columns of $W$ up to error $\mathcal{O} \left( \epsilon \, \kappa^2(W) \right)$, that is, the index set $\mathcal{K}$ identified by \ref{spa} satisfies \vspace{-0.1cm}
\[
\max_{1 \leq j \leq r} \min_{k \in \mathcal{K}} \left\|W(:,j) - \tilde{M}(:,k)\right\|_2 \leq \mathcal{O} \left( \epsilon \, \kappa^2(W) \right), \vspace{-0.1cm}
\]
where $\kappa(W) = \frac{\sigma_{\max}(W)}{\sigma_{\min}(W)}$ is the condition number of $W$. 
\end{theorem}

\noindent Moreover, \ref{spa} can be generalized by replacing the $\ell_2$ norm (step 2 of Algorithm~\ref{spa}) with any strongly convex function with a Lipschitz continuous gradient \cite{GV12}. 

\ref{spa} is closely related to several hyperspectral unmixing algorithms such as the automatic target generation process (ATGP)~\cite{RC03} and the successive volume maximization algorithm (SVMAX)~\cite{CM11}. 
It is also closely related to older techniques from other fields of research, in particular the modified Gram-Schmidt with column pivoting; 
 see, e.g., ~\cite{GV12, Ma14, Fal13, G14} and the references therein. 
Although \ref{spa} has many advantages (in particular, it is very fast and rather effective in practice), 
a drawback is that it requires the matrix $W$ to have rank~$r$. In fact, if $M$ is $r$-separable with $r < \rank(M)$, then \ref{spa} cannot extract enough columns even in noiseless conditions. Moreover, if the matrix $W$ is ill-conditioned, \ref{spa} will most likely fail even for very small noise levels (see Theorem~\ref{th1}).

\subsection{Contribution and Outline of the Paper}

The main contributions of this paper are 
\begin{itemize}

\item The introduction of a new fast and robust recursive algorithm for near-separable NMF, referred to as 
  the successive nonnegative projection algorithm (SNPA), which overcomes the drawback of \ref{spa} that the matrix $W$ has to be full column rank (Section~\ref{snpasec}). 
	
	\item The robustness analysis of \ref{snpa} (Section~\ref{robsnpa}). 
	First, we show that Theorem~\ref{th1} applies to \ref{snpa} as well, that is, we show that \ref{snpa} is robust to noise when $W$ has  full column rank.  
	Second, given a matrix $W$, we define a new parameter $\beta(W) \geq \sigma_{r}(W)$ which is in general positive even if $W$ has not full column rank. 
	We also define 
	$\kappa_{\beta}(W) = \frac{\max_i||W(:,i)||_2}{\beta(W)}$ and show that 
 \begin{quote} \textbf{Theorem~\ref{maincor}}. Let $\tilde{M}$ be a near-separable matrix satisfying Assumption~\ref{asssep} with \mbox{$\beta(W) > 0$}. If~$\epsilon \leq \mathcal{O} \left( \,  \frac{  \beta(W)  }{\kappa^3_{\beta}(W)} \right)$, then \ref{snpa} with $f(.) = ||.||_2^2$ identifies the columns of $W$ up to error $\mathcal{O} \left( \epsilon \, \kappa^3_{\beta}(W) \right)$. 
 \end{quote}
\noindent This proves that \ref{snpa} applies to a broader class of matrices ($W$ does not need to have full column rank). 
It also proves that \ref{snpa} is more robust than \ref{spa}: in fact, even when $W$ has rank~$r$, 
if $\frac{\sigma_{\min}(W)}{\sqrt{r} \kappa^2(W)} \ll \frac{  \beta(W)  }{\kappa^3_{\beta}(W)}$, then SNPA will outperform \ref{spa} 
as the noise level allowed by \ref{snpa} can be much larger. 
\end{itemize}

We illustrate the effectiveness of \ref{snpa} on several synthetic data sets and a real-world hyperspectral image in Section~\ref{ne}.  


\subsection{Notations} 

The unit simplex is defined as 
$\Delta^m = \left\{  x \in \mathbb{R}^m  \ \Big| \ x \geq 0, \sum_{i=1}^m x_i \leq 1 \right\}$, 
and the dimension $m$ will be dropped when it is clear from the context.  Given a matrix $W \in \mathbb{R}^{m \times r}$, $W(:,j)$, $W_{:j}$ or $w_j$ denotes its $j$th column. The zero vector is denoted $0$, its dimension will be clear from the context. We also denote $||W||_{1,2} = \max_{x, ||x||_1 \leq 1} ||Wx||_2 = \max_i ||W(:,i)||_2$. A matrix $W \in \mathbb{R}^{m \times r}$ is said to have full column rank if $\rank(W) = r$.

\section{Successive Nonnegative Projection Algorithm} \label{snpasec}

In this paper, we propose a new family of fast and robust recursive algorithms to solve near-separable NMF problems; see Algorithm \ref{snpa}. At each step of the algorithm, the column of the input matrix $\tilde{M}$ maximizing the function $f$ is selected, and then each column of $\tilde{M}$ is projected onto the convex hull of the columns extracted so far and the origin using the semi-metric induced by $f$. (A natural choice for the function $f$ in \ref{snpa} is $f(x) = ||x||_2^2$.) 
Hence the difference with \ref{spa} is the way the projection is performed. 
\renewcommand{\thealgorithm}{SNPA} 
\algsetup{indent=2em}
\begin{algorithm} 
\caption{Successive Nonnegative Projection Algorithm \label{snpa}}
\begin{algorithmic}[1] 
\REQUIRE Near-separable matrix $\tilde{M} = WH + N \in \mathbb{R}^{m \times n}$ satisfying Ass.~\ref{asssep} with \mbox{$\beta(W) > 0$},  the number $r$ of columns to be extracted, and a strongly convex function $f$ satisfying Ass.~\ref{fass1}. 
\ENSURE Set of indices $\mathcal{K}$ such that $\tilde{M}(:,\mathcal{K}) \approx W$ up to permutation. 
    \medskip 
		
\STATE Let $R = \tilde{M}$, $\mathcal{K} = \{\}$, $k = 1$.  
\WHILE {$R \neq 0$ and $k \leq r$}  
\STATE $p = \argmax_j f(R_{:j})$. $^\dagger$  \vspace{0.1cm} 
\STATE $\mathcal{K} = \mathcal{K} \cup \{p\}$. \vspace{0.1cm}  
\STATE $R(:,j) = \tilde{M}(:,j) - \tilde{M}(:,\mathcal{K})H^*(:,j)$ for all $j$,  
where \vspace{-0.25cm}
\[ 
H^*(:,j) = \argmin_{x \in \Delta} f\left(\tilde{M}(:,j) - \tilde{M}(:,\mathcal{K})x\right); \qquad \text{see Appendix~\ref{appA}.} \vspace{-0.5cm}
\] 
\STATE $k = k + 1$. 
\ENDWHILE 
\end{algorithmic}
$^\dagger$ In case of a tie, the index $j$ whose corresponding column of the original matrix $\tilde{M}$ maximizes $f$ is selected. In case of another tie, one of these columns is picked randomly. 
\end{algorithm} 
In this work, we perform the projections at step~5 of \ref{snpa} (which are convex optimization problems) using a fast gradient method, which is an optimal first-order method for minimizing convex functions with a Lipschitz continuous gradient \cite{Y04}; see Appendix~\ref{appA} for the implementation details.  
Although \ref{snpa} is computationally more expensive than \ref{spa}, it has the same asymptotic complexity, requiring a total of $\mathcal{O}(mnr)$ operations. 

SNPA is also closely related to the fast canonical hull algorithm, referred to as XRAY, from \cite{KSK12}. XRAY is a recursive algorithm for near-separable NMF and projects, at each step, the data points onto the convex cone of the columns extracted so far. 
The main differences between XRAY and \ref{snpa} are that 
\begin{enumerate} 
\item[(i)] XRAY uses another criterion to select a column of $M$ at each step. This is a crucial difference between SNPA and XRAY. In fact, it was discussed in \cite{KSK12} that in some cases (e.g., when a data point belongs to the cone spanned by two columns of $W$ and these two columns maximize the criterion simultaneously), XRAY may fail to identify a column of $W$ even in noiseless conditions; see the remarks on page 5 of \cite{KSK12}. 
This will be illustrated in Section~\ref{ne}.  (Note that there actually exists several variants of XRAY with different but closely related criteria for the selection step; however, they all share this undesirable property.) 

\item[(ii)] At each step, XRAY projects the data matrix onto the convex cone of the columns extracted so far while \ref{snpa} projects onto their convex hull (with the origin). 
In this paper, we will assume that the entries of each column of the matrix $H$ sum to at most one (equivalently that the columns of the data matrix belongs to the convex hull of the columns of $W$ and the origin); see Assumption~\ref{asssep} (and the ensuing discussion). 
Hence, performing the projection onto the convex hull allows to take this prior information into account.  
However, a variant of SNPA with projections onto the convex cone of the columns extracted so far is also possible although we have observed\footnote{We performed numerical experiments similar to that of Section~\ref{ne} and the variant of SNPA with projection onto the convex cones was less robust than \ref{snpa}, while being slightly more robust than XRAY (because of the difference in the selection criterion).} that, under Assumption~\ref{asssep}, it is less robust. 
It would be an interesting direction for further research to analyze this variant in details\footnote{Under Assumption~\ref{asssep}, the robustness analysis with projections onto the convex hull is made easier, and allowed us to derive better error bounds. The reason is that the columns of $H^*$ (see step 5 of Algorithm~\ref{snpa}) are normalized while an additional constant would be needed in the analysis (if we would follow exactly the same steps) to bound the norm of these columns if the projection was onto the convex cone.}.  
(Note that a variant of XRAY with projections onto convex hull does not work because the criterion used by XRAY in the selection step relies on the projections being performed onto the convex cone.)

\item[(iii)] XRAY performs the projection step with respect to the $\ell_2$ norm, while \ref{snpa} performs the projection with respect to the function $f$. 

\end{enumerate}

\section{Robustness of SNPA} \label{robsnpa}

In this section, we prove robustness of \ref{snpa} for any sufficiently small noise. The proofs are closely related to the robustness analysis of \ref{spa} developed in \cite{GV12}.

In Section~\ref{assdef}, we give the assumptions and definitions needed throughout the paper. 
In Section~\ref{noiseless}, we prove that \ref{snpa} identifies the columns of $W$ among the columns of $M$ exactly in the noiseless case, which explains the intuition behind \ref{snpa}. 
In Section~\ref{keylem}, we derive our key lemmas which allow us to show that the robustness analysis of \ref{spa} from Theorem~\ref{th1} (which requires $W$ to be full column rank) also applies to \ref{snpa}; see Theorems~\ref{threc} and~\ref{th4}. 
In Section~\ref{genmat}, we generalize the analysis to a broader class of matrices for which $W$ is not required to be full column rank; see Theorems~\ref{mazette} and~\ref{maincor}.  

In Sections~\ref{improve} and~\ref{choicef}, we briefly discuss some possible improvements of \ref{snpa}, and the choice of the function $f$,  respectively.

\subsection{Assumptions and Definitions} \label{assdef}

In this section, we describe the assumptions and definitions useful to prove robustness of \ref{snpa}.

Without loss of generality, we will assume throughout the paper that the input matrix has the following form: 
\begin{assumption}[Near-Separable Matrix] \label{asssep}
The separable matrix $M \in \mathbb{R}^{m \times n}$ can be written as 
\[
M = W H = W [I_r, H'], 
\]
where $W \in \mathbb{R}^{m \times r}$, 
$H \in \mathbb{R}^{r \times n}_+$, and $H(:,j) \in \Delta$ for all $j$. The near-separable matrix $\tilde{M}$ is given by $\tilde{M} = M + N$ where $N$ is the noise with $||N||_{1,2} \leq \epsilon$. 
\end{assumption}  

\noindent Any nonnegative near-separable matrix $M = WH = W[I_r, H']\Pi$ (with $W, H' \geq 0$ and $\Pi$ a permutation) can be put in this form by proper permutation and normalization of its columns. 
In fact, permuting the columns of $M$ and $H$ so that the first $r$ columns of $M$ correspond to the columns of $W$, we have $M = W [I_r, H'] = WH$. 
The permutation does not affect our analysis because \ref{snpa} is not sensitive to permutation, 
while it makes the presentation nicer (the permutation matrix $\Pi$ can be discarded).  
For the scaling, we  
(i)~divide each column of $M$ by its $\ell_1$ norm (unless it is an all-zero column in which case we discard it) and divide the corresponding column of $H$ by the same constant (hence we still have $M = WH$), and 
(ii)~divide each column of $W$ by its $\ell_1$ norm and multiply the corresponding row of $H$ by the same constant (hence $WH$ is unchanged). Since the entries of each column of the normalized matrices $M$ and $W$ sum to one, and $M(:,j) = W H(:,j)$ for all $j$, the entries of each column of $H$ must also sum to one: for all $j$, 
\[
1 = \sum_i M(i,j) = \sum_i \sum_k W(i,k) H(k,j) = \sum_k H(k,j) \sum_i  W(i,k) = \sum_k H(k,j); 
\]
see also the discussion in \cite{GV12}. 
Column normalization also makes the presentation nicer: in fact, otherwise the noise that can be tolerated on each column of $\tilde{M} = M + N$ will have to be proportional to the norm of the corresponding column of the matrix $H$ (for example, an all-zero column cannot tolerate any noise because it can be made an extreme ray of the cone spanned by the columns of $M$ for any positive noise level). 

Note that Assumption~\ref{asssep} \emph{does not require $W$ to be nonnegative} hence our result will apply to a broader class than the nonnegative near-separable matrices. 
It is also interesting to note that data matrices corresponding to hyperspectral images are naturally scaled since the columns of $H$ correspond to abundances and their entries sum to one (see Section~\ref{nsnmf} for more details, and Section~\ref{hsi} for some numerical experiments).

We will also assume that, in \ref{snpa}, 
\begin{assumption} \label{fass1} 
The function $f:\mathbb{R}^m \to \mathbb{R}_+$ is strongly convex with parameter $\mu > 0$, its gradient is Lipschitz continuous with constant $L$, and its global minimizer is the all-zero vector with $f(0) = 0$. 
\end{assumption}

A function $f$ is strongly convex with parameter $\mu$ if and only if it is convex and for any $x, y \in \text{dom}(f)$ and for all $\delta \in [0,1]$ 
\begin{equation} \label{strcon}
f(\delta x + (1-\delta) y) 
\; \leq \;  
\delta f(x) + (1-\delta) f(y) - \frac{\mu}{2} \delta   (1-\delta) ||x-y||_2^2. 
\end{equation} 
Moreover, its gradient is Lipschitz continuous with constant $L$ if and only if for any $x, y \in \text{dom}(f)$, we have $||\nabla f (x) - \nabla f (y) ||_2 \leq L ||x-y||_2$. 
Convex analysis also tells us that if $f$ satisfies Assumption~\ref{fass1} then, for any $x, y$, 
\[
f(x) + \nabla f(x)^T (y-x) + \frac{\mu}{2} ||x-y||_2^2 
\quad \leq \quad
f(y) \quad  \leq  \quad 
f(x) + \nabla f(x)^T (y-x) + \frac{L}{2} ||x-y||_2^2. 
\]
In particular, taking $x=0$, we have, for any $y \in \mathbb{R}^m$, 
\begin{equation} \label{normconv}
 \frac{\mu}{2} ||y||_2^2 
 \quad \leq  \quad 
f(y)
 \quad \leq  \quad 
\frac{L}{2} ||y||_2^2, 
\end{equation}
since $f(0) = 0$ and $\nabla f(0) = 0$ (because zero is the global minimizer of~$f$); see, e.g., \cite{UL01}. Note that this implies $f(x) > 0$ for any $x \neq 0$ hence $f$ induces a semi-metric; the distance between two points $x$ and $y$ being defined by $f(x-y)$. \\

We will use the following notation for the residual computed at step~5 of Algorithm~\ref{snpa}. 
\paragraph{Definiton (Projection and Residual)} 
Given $B \in \mathbb{R}^{m \times s}$ and a function $f$ satisfying Assumption~\ref{fass1}, we define 
 the projection $\mathcal{P}_B^f(x)$ of $x$ onto the convex hull of the columns of $B$ with respect to the semi-metric induced by $f(.)$ as follows: 
\[
\mathcal{P}_B^f(x):  \mathbb{R}^{m} \rightarrow \mathbb{R}^{m} : x \rightarrow \mathcal{P}_B^f(x) = By^* , \; \text{ where }  y^* = \argmin_{y \in \Delta} f(x - By). 
\]   
We also define the residual $\mathcal{R}_B^f$ of the projection $\mathcal{P}_B^f$ as follows: 
\[
\mathcal{R}_B^f :  \mathbb{R}^{m} \rightarrow \mathbb{R}^{m} : x \rightarrow \mathcal{R}_B^f(x) = x - \mathcal{P}_B^f(x). 
\] 
For a matrix $A \in \mathbb{R}^{m \times r}$, we will denote $\mathcal{P}_B^f(A)$ the matrix whose columns are the projections of the columns of $A$, that is, $\left(\mathcal{P}_B^f(A)\right)_{:i} = \mathcal{P}_B^f(A_{:i})$ for all $i$, and $\mathcal{R}_B^f(A) = A - \mathcal{P}_B^f(A)$.

Given a matrix $W \in \mathbb{R}^{m \times r}$, we introduce the following  notations: 
\begin{align*}
\alpha(W) & = 
\min_{1 \leq j \leq r, x \in \Delta} \left\|  W(:,j) - W(:,\mathcal{J})x \right\|_2 \text{ where } \mathcal{J} = \{1,2,\dots,r\} \backslash \{j\}, \\
\nu(W) & =  \min_{i} ||w_i||_2, \\
\gamma(W) & = \min_{i\neq j} ||w_i-w_j||_2, \\ 
\omega(W) & = \min \left\{\nu(W), \frac{1}{\sqrt{2}} \gamma(W)\right\}, \\ 
K(W) & = ||W||_{1,2} = \max_{i} ||w_i||_2, \quad  \text{and} \\ 
\sigma(W) & = \left\{ 
\begin{array}{cc} 
 \sigma_r(W) = \sigma_{\min}(W) & \text{ if $m \geq r$}, \\
 0 & \text{ if $m < r$.} \\ 
\end{array} 
\right. 
\end{align*}

The parameter $\alpha(W)$ is the minimum distance between a column of $W$ and the convex hull of the other columns of $W$ and the origin.  
It is interesting to notice that, under Assumption~\ref{asssep}, $\alpha(W) > 0$ is a  necessary condition to being able to identify the columns of $W$ among the columns of $M$ (in fact, $\alpha(W) = 0$ means that a column of $W$ belongs to the convex hull of the other columns of $W$ and the origin hence cannot be distinguished from the other data points). It is also a sufficient condition as some algorithms are guaranteed to identify the columns of $W$ when $\alpha(W) > 0$ even in the presence of noise \cite{AGKM11, G13, GL13}.

\subsection{Recovery in the Noiseless Case} \label{noiseless}

In this section, we show that, in the noiseless case, \ref{snpa} is able to perfectly identify the columns of $W$ among the columns of $M$. Although this result is implied by our analysis in the noisy case (see Section~\ref{fullrank}), it gives the intuition behind the working of \ref{snpa}.

\begin{lemma} \label{projn1}
 Let $B \in \mathbb{R}^{m \times s}$, $A \in \mathbb{R}^{m \times k}$, $z \in \Delta^k$, and $f$ satisfy Assumption~\ref{fass1}. Then 
\[
f\left( \mathcal{R}_B^f(Az) \right) \leq f\left( \mathcal{R}_B^f(A) z \right). 
\]
\end{lemma} 
\begin{proof} 
Let us denote $Y(:,j)= \argmin_{y \in \Delta} f( A(:,j) - By )$ for all  $j$, that is, $\mathcal{R}_B^f(A) = A - BY$. We have 
\begin{align*}
f\left( \mathcal{R}_B^f(Az) \right) & = \min_{y \in \Delta} f\left(Az - By \right)  \leq f\left(Az - BYz \right) =  f\left(  \mathcal{R}_B^f(A) z\right) .
\end{align*} 
The inequality follows from $Yz \in \Delta$, since $Y(:,j) \in \Delta$ $\forall j$ and $z \in \Delta$. 
\end{proof}

\begin{theorem} Let $M = W \, [I_r, \, H'] = WH$ be a separable matrix satisfying Assumption~\ref{asssep} where $W$ has full column rank, and let $f$ satisfy Assumption~\ref{fass1}. 
Then \ref{snpa} applied on matrix $M$ identifies a set of indices $\mathcal{K}$ such that, up to permutation,  $M(:,\mathcal{K}) = W$. 
\end{theorem}
\begin{proof} 
We prove the result by induction. 

\noindent \emph{First step.} Since the columns of $M$ belong to the convex hull of the columns of $W$ and the origin (the entries of each column of $H$ are nonnegative and sum to at most one), and since a strongly convex function is always maximized at a vertex of a polytope, 
a column of $W$ will be identified at the first step of SNPA (the origin cannot be extracted since, by assumption, it minimizes $f$). More formally, let $h \in \Delta^r$, we have 
\begin{align*} 
f \left( Wh \right) 
& = f\left(   \sum_{k=1}^r   W(:,k) h(k) + \left(1- \sum_{k=1}^r h(k)\right) 0 \right) \\
& \leq \sum_{k=1}^r h(k) f\left(   W(:,k)  \right) \\
& \leq \max_k f\left(   W(:,k)  \right). 
\end{align*}
The first inequality follows from convexity of $f$ and the fact that $f(0) = 0$. 
By strong convexity, see Equation~\eqref{strcon}, the first inequality is always strict unless $h = e_j$ for some $j$ (where $e_j$ is the $j$th column of the identity matrix). 
The second inequality follows from $h \in \Delta$ and the fact that $f(x) > 0$ for any $x \neq 0$. 
Since all columns of $M$ can be written as $Wh$ for some $h \in \Delta^r$, this implies that, at the first step, \ref{snpa} extracts the index corresponding to the column of $W$ maximizing $f$. 

\noindent \emph{Induction step.} Assume \ref{snpa} has extracted some indices $\mathcal{K}$ corresponding to columns of $W$, that is, 
$M(:,\mathcal{K}) = W(:,\mathcal{I})$ for some $\mathcal{I}$. We have for any $h \in \Delta^r$ that 
\begin{align*}
f\left( \mathcal{R}_{W(:,\mathcal{I})}^f(Wh)  \right) 
& \hspace{0.54cm}  \underset{(\text{Lemma~\ref{projn1}})}{\leq}  \hspace{0.54cm} 
f\left(  \mathcal{R}_{W(:,\mathcal{I})}^f(W) h \right) \\
& 
\hspace{0.88cm}  \underset{(\text{Ass.~\ref{fass1}})}{\leq}  \hspace{0.88cm} 
\sum_{k=1}^r h(k) f\left(  \mathcal{R}_{W(:,\mathcal{I})}^f(W(:,k)) \right) \\
& 
 \underset{({h \in \Delta^r}, f(x) > 0 \, \forall x \neq 0)}{\leq} 
\max_k f\left(  \mathcal{R}_{W(:,\mathcal{I})}^f(W(:,k)) \right). 
\end{align*} 
Finally, noting that the residual $R$ in \ref{snpa} is equal to $\mathcal{R}_{W(:,\mathcal{I})}(M)$ and since 
\begin{itemize}
 \item $\mathcal{R}_{W(:,\mathcal{I})}^f(W(:,k)) = 0$ for all $k \in \mathcal{I}$, 
 \item $\mathcal{R}_{W(:,\mathcal{I})}^f(W(:,k)) \neq 0$ for all $k \notin \mathcal{I}$ because $W$ has full column rank, and  
 \item the second inequality is strict unless $h \neq e_j$ for some $j$ by strong convexity~of~$f$, 
\end{itemize}
SNPA identifies a column of $W$ not extracted yet. 
\end{proof} 
\noindent Note that the proof does not need $W$ to be full column rank, but only requires that  \mbox{$\mathcal{R}_{W(:,\mathcal{I})}^f(W(:,k)) \neq 0$} for all $k \notin \mathcal{I}$ for any subset $\mathcal{I}$ of $\{ 1, 2, \dots, r\}$.  This observation will be exploited in Section~\ref{genmat} to show robustness of \ref{snpa} when $W$ is not full column rank.

\subsection{Key Lemmas} \label{keylem}

In the following, we derive the key lemmas to prove robustness of \ref{snpa}.  

More precisely, we subdivide the columns of $W$ into two subsets as follows $W = [A, B]$. 
The columns of the matrix $B$ represent the columns of $W$ which have already been approximately identified by \ref{snpa} while the columns of $A$ are the columns of $W$ yet to be identified. 
The columns of the matrix $\tilde{B}$ correspond to the columns of matrix $\tilde{M}$ already extracted by \ref{snpa}, and we will assume that $||B - \tilde{B}||_{1,2} \leq \bar{\epsilon}$ for some constant $\bar{\epsilon}$. 
Lemmas~\ref{lem1} to \ref{sig2} lead to a lower bound for $\omega\left(\mathcal{R}^f_{\tilde{B}}(A)\right)$ using $\sigma([A,B])$; see Corollary~\ref{omeg}. 
Combined with Lemmas~\ref{projn} to~\ref{lemscf}, these lemmas will imply that if $W$ has full column rank then a column of $W$ not extracted yet (that is, a column of $A$) is identified approximately at the next step of SNPA; see Theorem~\ref{Th2}. Finally, using that result inductively leads to the robustness of SNPA; see Theorem~\ref{threc}.

\begin{lemma} \label{lem1}
For any $B \in \mathbb{R}^{m \times s}$, $x \in \mathbb{R}^{m}$, and $f$ satisfying Assumption~\ref{fass1}, we have 
\[
 \left\| \mathcal{R}_B^f(x) \right\|_2 
 \leq  
 \sqrt{\frac{L}{\mu}} \left\| x \right\|_2  . 
\]
\end{lemma} 
\begin{proof} Using Equation~\eqref{normconv}, we have 
\begin{align*}
 \left\| \mathcal{R}_B^f(x) \right\|_2^2  
& \leq \frac{2}{\mu} f\left(\mathcal{R}_B^f(x)\right) 
= \frac{2}{\mu} \min_{y \in \Delta}  f(x - By) 
\leq \frac{L}{\mu} \min_{y \in \Delta}  ||x - By||_2^2 
\leq \frac{L}{\mu} ||x||_2^2, 
\end{align*}
since $0 \in \Delta$. 
\end{proof}

\begin{lemma} \label{lemaa} 
Let $B \in \mathbb{R}^{m \times s}$ and $B = \tilde{B} + N$ with $||N||_{1,2} \leq \bar{\epsilon}$, and let $f$ satisfy Assumption~\ref{fass1}. 
Then,  
\[
 \max_j \left\| \mathcal{R}_{\tilde{B}}^f(b_j) \right\|_2 
 \leq
 \sqrt{\frac{L}{\mu}}  \bar{\epsilon}  . 
\]
\end{lemma} 
\begin{proof}  Using Equation~\eqref{normconv}, we have, for all $j$, 
\begin{align*}
\left\| \mathcal{R}_{\tilde{B}}^f(b_j) \right\|_2^2  
& \leq \frac{2}{\mu} f\left( \mathcal{R}_{\tilde{B}}^f(b_j)   \right)  
 = \frac{2}{\mu} \min_{x \in \Delta} f( b_j - \tilde{B} x   ) \\ 
& \leq \frac{2}{\mu} f( b_j - \tilde{b}_j   ) = \frac{2}{\mu} f( n_j ) \\ 
& \leq \frac{L}{\mu} || n_j ||_2^2 \leq \frac{L}{\mu} \bar{\epsilon}^2.  
\end{align*}

\end{proof}

\begin{lemma} \label{nuf}
Let $A \in \mathbb{R}^{m \times k}$, $B \in \mathbb{R}^{m \times s}$, and $f$ satisfy Assumption~\ref{fass1}. Then, 
\[
\nu \left( \mathcal{R}_B^f(A) \right) \geq  \alpha([A, B]). 
\]
\end{lemma} 
\begin{proof}
This follows directly from the definitions of $\alpha$ and $\mathcal{R}_B^f$: in fact, 
\begin{align*}
\nu \left( \mathcal{R}_B^f(A) \right)  
& \geq \min_j \min_{y \in \Delta}|| A(:,j) - B y ||_2  \geq \alpha([A,B]).    
\end{align*}
\end{proof}

\begin{lemma} \label{lemA}
Let $Z$ and $\tilde{Z} \in \mathbb{R}^{m \times r}$ satisfy $||Z - \tilde{Z}||_{1,2} \leq \bar{\epsilon}$. 
Then, 
\[
\alpha(\tilde{Z}) \geq \alpha(Z) - 2 \bar{\epsilon} . 
\]
\end{lemma} 
\begin{proof} 
Denoting $N = Z -  \tilde{Z}$ and $\mathcal{J} = \{1,2,\dots,k\} \backslash\{j\}$, we have
\begin{align*}
\alpha(\tilde{Z}) 
& = \min_{1 \leq j \leq k, x \in \Delta} ||\tilde{z}_j - \tilde{Z}(:,\mathcal{J})x ||_2 \\
& = \min_{1 \leq j \leq k, x \in \Delta} ||{z}_j - n_j - {Z}(:,\mathcal{J})x + N(:,\mathcal{J}) x||_2 \\ 
& \geq \min_{1 \leq j \leq k, x \in \Delta} ||{z}_j - {Z}(:,\mathcal{J})x ||_2 - ||n_j||_2 - ||N(:,\mathcal{J})x||_2 \\ 
& \geq \min_{1 \leq j \leq k, x \in \Delta} ||{z}_j - {Z}(:,\mathcal{J})x ||_2 - 2 \bar{\epsilon} \\ 
& = \alpha({Z})  - 2 \bar{\epsilon},  
\end{align*}
since $\max_{x \in \Delta} ||N(:,\mathcal{J})x||_2 \leq \max_{||x||_1 \leq 1} ||N(:,\mathcal{J})x||_2 =  ||N(:,\mathcal{J})||_{1,2} \leq \bar{\epsilon}$.
\end{proof}

\begin{corollary} \label{nulem}
Let $A \in \mathbb{R}^{m \times k}$, $B$ and $\tilde{B} \in \mathbb{R}^{m \times s}$ satisfy $||B - \tilde{B}||_{1,2} \leq \bar{\epsilon}$, and $f$ satisfy Assumption~\ref{fass1}. 
Then, 
\[
\nu \left( \mathcal{R}_{\tilde{B}}^f(A) \right) 
\geq  \Big( \alpha \left( [A, {B}] \right) -  \min(s,2) {\bar{\epsilon}} \Big) . 
\] 
\end{corollary} 
\begin{proof} 
If $s = 0$, the result follows from Lemma~\ref{nuf} ($B$ is an empty matrix). If $s = 1$, then it is easily derived using the same steps as in the proof of Lemma~\ref{lemA} ($B$ only has one column). 
Otherwise, Lemmas~\ref{nuf} and~\ref{lemA} imply that 
\[
\nu \left( \mathcal{R}_{\tilde{B}}^f(A) \right) 
\geq 
\alpha \left( [A, \tilde{B}] \right) 
\geq 
\alpha \left( [A, {B}] \right) - 2 {\bar{\epsilon}}. 
\] 
\end{proof}

\begin{lemma} \label{alsi}
For any $W \in \mathbb{R}^{m \times r}$,  
$\alpha(W) \geq \sigma(W)$. 
\end{lemma}
\begin{proof}
We have 
\begin{align*} 
\alpha(W)  
& = \min_{1 \leq j \leq r} \min_{x \in \Delta} ||W(:,j) - W(:,\mathcal{J}) x||_2 \\ 
& \geq \min_{1 \leq j \leq r} \min_{z \in \mathbb{R}^r, z(j) = 1} ||W z||_2 \\ 
& \geq  \min_{z \in \mathbb{R}^r, ||z||_2 \geq 1} ||W z||_2 = \sigma(W). 
\end{align*} 
\end{proof}

\begin{lemma} \label{sig}
Let $x, y \in \mathbb{R}^{m}$, $B \in \mathbb{R}^{m \times s}$, and $f$ satisfy Assumption~\ref{fass1}. Then 
\[
\left\| \mathcal{R}_B^f(x) \right\|_2
\geq  \sigma([B, x]) 
\quad \text{ and } \quad
\left\| \mathcal{R}_B^f(x) - \mathcal{R}_B^f(y) \right\|_2
\geq 
 \sqrt{2} \; \sigma([B, x, y]). 
\]
\end{lemma} 
\begin{proof}
Let us denote $z_x = \argmin_{z \in \Delta} f(x-Bz)$ and $z_y = \argmin_{z \in \Delta} f(y-Bz)$, we have 
\begin{align*}
|| \mathcal{R}_B^f(x) ||_2 
 = || x-Bz_x ||_2   
 \geq \min_{z \in \mathbb{R}^{p}} ||x + B z||_2 
 & = \min_{z \in \mathbb{R}^{p+1}, z(1) = 1} ||[x ,B] z||_2 \\
 & \geq \min_{||z||_2 \geq 1} ||[x, B] z||_2 \geq \sigma([x,B]) , 
\end{align*} 
and
\begin{align*}
|| \mathcal{R}_B^f(x) - \mathcal{R}_B^f(y) ||_2 
& = || (x-Bz_x) - (y-Bz_y) ||_2  \\ 
& \geq \min_{z \in \mathbb{R}^{p}} ||x - y + B z||_2 \\ 
& = \min_{z \in \mathbb{R}^{p+2}, z(1) = 1, z(2) = -1} ||[x, y ,B] z||_2 \\ 
& \geq \min_{||z||_2 \geq \sqrt{2}} ||[x, y, B] z||_2 = \sqrt{2} \, \sigma([x,y,B]) . 
\end{align*} 
\end{proof}

\begin{lemma}[Singular Value Perturbation \cite{GV96}, Cor.~8.6.2] \label{weyl} 
Let $\tilde{B} = B + N \in \mathbb{R}^{m \times s}$ with $s \leq m$. Then, for all $1 \leq i \leq s$, 
\[
\left| \sigma_i(B) - \sigma_i(\tilde{B}) \right| \; \leq \; \sigma_{\max}(N) = ||N||_2 \; \leq \; 
\sqrt{s} \; ||N||_{1,2}. 
\]
\end{lemma}

\begin{lemma} \label{sig2}
Let $x, y \in \mathbb{R}^{m}$, $B$ and $\tilde{B} \in \mathbb{R}^{m \times s}$ be 
such that $||\tilde{B} - B||_{1,2} \leq \bar{\epsilon}$, and $f$ satisfy Assumption~\ref{fass1}. 
Then 
\[
\left\| \mathcal{R}_{\tilde{B}}^f({x}) - \mathcal{R}_{\tilde{B}}^f({y}) \right\|_2
\geq 
\sqrt{2} \, \left( \sigma([B, x, y]) - \sqrt{s} \; \bar{\epsilon} \right) . 
\]
\end{lemma} 
\begin{proof}
This follows from Lemmas~\ref{sig} and \ref{weyl}. 
\end{proof}

\begin{corollary} \label{omeg}
Let $A \in \mathbb{R}^{m \times k}$, $B$ and $\tilde{B} \in \mathbb{R}^{m \times s}$ satisfy $||\tilde{B} - B||_{1,2} \leq \bar{\epsilon}$, and $f$ satisfy Assumption~\ref{fass1}. Then, 
\[
\omega \left( \mathcal{R}_{\tilde{B}}^f(A) \right) 
\geq  \Big( \sigma \left( [A, {B}] \right) - \sqrt{2 s} \; {\bar{\epsilon}} \Big) , 
\] 
\end{corollary} 
\begin{proof}
Using Lemma~\ref{sig2}, we have 
\begin{align*}
\frac{1}{\sqrt{2}} \gamma \left( \mathcal{R}_{\tilde{B}}^f(A) \right) 
& = \frac{1}{\sqrt{2}}\min_{i \neq j} ||\mathcal{R}_{\tilde{B}}^f(a_i) - \mathcal{R}_{\tilde{B}}^f(a_j)||_2 
 \geq \sigma([B, a_i, a_j]) - \sqrt{s} \; \bar{\epsilon}   
 \geq \sigma([A, B]) - \sqrt{s} \; \bar{\epsilon}. 
\end{align*}
Using Corollary~\ref{nulem} and Lemma~\ref{alsi}, we have  
\begin{align*}
\nu \left( \mathcal{R}_{\tilde{B}}^f(A) \right)  
 & = \min_{i} \left\|\mathcal{R}_{\tilde{B}}^f(a_i)\right\|_2 
 \geq  \alpha \left( [A, {B}] \right) -  \min(s,2) {\bar{\epsilon}} 
 \geq  \sigma \left( [A, {B}] \right) -  \min(s,2) {\bar{\epsilon}} . 
\end{align*}
Since $\sqrt{2s} \geq \min(s,2)$ for any $s \geq 0$, the proof is complete. 
\end{proof}

\begin{lemma} \label{projn}
 Let $B \in \mathbb{R}^{m \times s}$, $A \in \mathbb{R}^{m \times k}$, $n \in \mathbb{R}^{m}$, $z \in \Delta^k$, and $f$ satisfy Assumption~\ref{fass1}. Then 
\[
f\left( \mathcal{R}_B^f(Az + n) \right) \leq f\left( \mathcal{R}_B^f(Az) + n \right) 
\quad \text{ and } \quad
 f\left( \mathcal{R}_B^f(Az + n) \right) \leq  f\left( \mathcal{R}_B^f(A) z + n \right). 
\]
\end{lemma} 
\begin{proof} 
Let us denote $y^* = \argmin_{y \in \Delta} f( Az - By )$ and 
$Y(:,j)= \argmin_{y \in \Delta} f( A(:,j) - By )$ for all $j$,  
that is, $\mathcal{R}_B^f(A) = A - BY$. We have 
\begin{align*}
f\left( \mathcal{R}_B^f(Az + n) \right) 
& = \min_{y \in \Delta} f\left(Az  + n - By \right)  
\leq f\left(Az - By^* + n \right) =  f\left(  \mathcal{R}_B^f(Az)  + n \right) , 
\end{align*} 
and 
\begin{align*}
f\left( \mathcal{R}_B^f(Az + n) \right) 
& = \min_{y \in \Delta} f\left(Az  + n - By \right)  
 \leq f\left(Az - BYz + n \right) =  f\left(  \mathcal{R}_B^f(A) z + n \right) ,
\end{align*}
where the inequality follows from $y = Yz \in \Delta$, since \mbox{$Y(:,j) \in \Delta$} $\forall j$ and $z \in \Delta$. 
\end{proof}

Let us also recall two useful lemmas from \cite{GV12}. 

\begin{lemma}[\cite{GV12}, Lemma 3] \label{fbound} 
Let the function $f$ satisfy Assumption~\ref{fass1}.  
Then, for any $||x||_2 \leq K$ and $||n||_2 \leq \epsilon \leq K$, we have 
\begin{equation} 
f(x) - \epsilon K L \leq f(x+n) 
\leq f(x) + \frac{3}{2} \epsilon K L. \nonumber 
\end{equation} 
\end{lemma}

\begin{lemma}[\cite{GV12}, Lemma 2] \label{lemscf}
Let $Z = [P, Q]$ where $P \in \mathbb{R}^{m \times k}$ and $Q \in \mathbb{R}^{m \times s}$, and let $f$ satisfy Assumption~\ref{fass1}. 
If $\nu(P) > 2 \sqrt{\frac{L}{\mu}}  K(Q)$, 
then, for any $0 \leq \delta \leq \frac{1}{2}$, 
\[ 
f^* \quad = \quad \max_{x \in \Delta} f(Zx) \; \text{ such that }  x_i \leq 1 - {\delta} \, \text{ for } 1 \leq i \leq k, 
\] 
satisfies 
\begin{align*}
f^* 
& \leq \max_i f(p_i) -  \frac{1}{2} \, \mu \,  (1-\delta) \, \delta \, \omega(P)^2  .  
\end{align*}
Moreover, the maximum is attained only at point $x$ such that $x_i = 1 - {\delta} \, \text{ for some } 1 \leq i \leq k$. 
\end{lemma}

\subsection{Robustness of \ref{snpa} when $W$ has Full Column Rank} \label{fullrank}

Theorem~\ref{Th2} below shows that if \ref{snpa} has already extracted some columns of $W$ up to error $\bar{\epsilon}$, then the next extracted column of $\tilde{M}$ will be close to a column of $W$ not extracted yet. This will allow us to prove inductively that \ref{snpa} is robust to noise; 
see Theorem~\ref{threc}. 

\begin{theorem} \label{Th2} 
Let 
\begin{itemize} 

\item $f$ satisfy Assumption~\ref{fass1}, with  strong convexity parameter $\mu$, and its gradient have Lipschitz constant $L$.  

\item $\tilde{M}$ satisfy Assumption~\ref{asssep} with $\tilde{M} = M + N = WH + N$, $W = [A, \, B]$, $A \in \mathbb{R}^{m \times k}$, $B \in \mathbb{R}^{m \times s}$, $||N||_{1,2} \leq \epsilon$, and $H = [I_r, H'] \in \mathbb{R}^{r \times n}_+$ where $H(:,j) \in \Delta$ for all $j$. 
 
\item $\tilde{B} \in \mathbb{R}^{m \times s}$ satisfy 
\[
||B - \tilde{B}||_{1,2} \leq \bar{\epsilon} = C \epsilon, \quad \text{ for some $C \geq 0$}.
\]

\item $W = [A, B]$ be such that $\sigma(W) = \sigma > 0$. We denote $\alpha = \alpha(W)$, and $K = K(W)$. 
 
\item $\epsilon$ be sufficiently small so that 
\[
\epsilon < \min \left( \frac{{\sigma}^2 \mu^{3/2}}{144 K L^{3/2}}, \frac{\alpha \mu}{4 L C}, \frac{\sigma}{2 C \sqrt{2 s}} \right). 
\]

\end{itemize}
Then the index $i$ corresponding to a column $\tilde{m}_i$ of $\tilde{M}$ that maximizes the function $f\left(\mathcal{R}^f_{\tilde{B}}(.)\right)$  satisfies 
\begin{equation} \label{eqmdel2} 
m_i = Wh_i = [A, B]h_i, 
\;  \text{ where } h_i(\ell) \geq 1 - \delta \text{ with }  1 \leq \ell \leq k, 
\end{equation} 
and $\delta =  \frac{72  \epsilon K L^{3/2}}{\sigma^2 \mu^{3/2}}$, 
which implies  
\begin{equation} \label{errbound}
||\tilde{m}_i - w_{\ell}||_2 = ||\tilde{m}_i - a_{\ell}||_2  
\leq \epsilon + 2 K \delta
 = \epsilon \left( 1 + 144 \frac{K^2}{\sigma^2} \frac{L^{3/2}}{\mu^{3/2}} \right).   
\end{equation} 
\end{theorem}
\begin{proof}
First note that $\epsilon \leq \frac{\sigma^2 \mu^{3/2}}{144 K L^{3/2}}$ implies $\delta \leq \frac{1}{2}$. 
Then, let us show that 
\[
\nu \left( \mathcal{R}^f_{\tilde{B}}(A) \right) 
> 2 \sqrt{\frac{L}{\mu}} K \left( \mathcal{R}^f_{\tilde{B}}(B) \right) , 
\]
so that Lemma~\ref{lemscf} will apply to $P = \mathcal{R}^f_{\tilde{B}}(A)$ and $Q = \mathcal{R}^f_{\tilde{B}}(B)$. 
Since $||B- \tilde{B}||_{1,2} \leq \bar{\epsilon}$, by Lemma~\ref{lemaa}, we have $K \left( \mathcal{R}^f_{\tilde{B}}(B) \right) \leq \sqrt{\frac{L}{\mu}} \bar{\epsilon}$. Therefore, 
\begin{align*}
\nu \left( \mathcal{R}^f_{\tilde{B}}(A) \right) 
& \hspace{0.3cm} \underset{(\text{Corollary~\ref{nulem}})}{\geq}  \alpha - 2 \bar{\epsilon} \\ 
& \hspace{0cm}\underset{\left(\bar{\epsilon} = C \epsilon  < \frac{\alpha \mu}{4 L}, L \geq \mu\right)}{>} \hspace{0cm} 
2 \frac{L}{\mu}  \bar{\epsilon} \\ 
& \hspace{0.45cm} \underset{(\text{Lemma~\ref{lemaa}})}{\geq}  \hspace{0.4cm} 2 \sqrt{\frac{L}{\mu}} K \left( \mathcal{R}^f_{\tilde{B}}(B) \right). 
\end{align*} 
Let us also show that $\omega \left( \mathcal{R}^f_{\tilde{B}}(A) \right) 
\geq \frac{\sigma}{2}$.  
By Corollary~\ref{omeg} and the assumption that $\bar{\epsilon} = C \epsilon \leq \frac{\sigma}{2 \sqrt{2s}}$, we have 
\begin{equation}  \label{epssi} 
\omega \left( \mathcal{R}^f_{\tilde{B}}(A) \right) 
\geq 
\sigma -  \sqrt{2s} \bar{\epsilon} 
\geq \frac{\sigma}{2}. 
\end{equation}

We can now prove Equation~\eqref{eqmdel2} by contradiction. Assume the extracted index, say the $i$th, which maximizes 
$f\left(\mathcal{R}^f_{\tilde{B}}(.)\right)$ among the columns of $\tilde{M}$,  satisfies $\tilde{m}_i = m_i + n_i = Wh_i + n_i$ with $h_i(\ell) < 1 - \delta$ for $1 \leq \ell \leq k$.  
We have 
\begin{align}
f \left( \mathcal{R}^f_{\tilde{B}} (\tilde{m}_i) \right) 
& \hspace{0.40cm} \underset{(\text{Lemma~\ref{projn}})}{\leq} \hspace{0.40cm} 
f\left( \mathcal{R}^f_{\tilde{B}} (W) h_i + n_i \right)  \nonumber \\
& \hspace{0.40cm} \underset{(\text{Lemma~\ref{fbound}})}{\leq} \hspace{0.40cm} 
f\left(\mathcal{R}^f_{\tilde{B}} (W) h_i\right)  + \frac{3}{2} \epsilon K\left( \mathcal{R}^f_{\tilde{B}}(A) \right) L \nonumber \\ 
& \hspace{0.68cm} \underset{(\text{Ass.~\ref{fass1}})}{<}  \hspace{0.68cm} 
\max_{x \in \Delta^r, x(\ell) \leq 1-\delta \, 1 \leq \ell \leq k} f\left(\mathcal{R}^f_{\tilde{B}} (W) x\right)  + \frac{3}{2} \epsilon K \sqrt{\frac{L}{\mu}}  L  \nonumber \\ 
& \hspace{0.40cm} \underset{(\text{Lemma~\ref{lemscf}})}{\leq} \hspace{0.40cm}
 \max_j f\left(\mathcal{R}^f_{\tilde{B}} (a_j)\right)  
- \frac{1}{2} \mu \delta (1-\delta) \omega\left( \mathcal{R}^f_{\tilde{B}}(A) \right)^2   + \frac{3}{2} \epsilon K {\frac{L^{3/2}}{\mu^{1/2}}}    \nonumber \\
& \hspace{0.07cm} \underset{(\text{Lem.~\ref{projn},\, Eq.~\ref{epssi}})}{\leq} \hspace{0.07cm} 
\max_j f\left(\mathcal{R}^f_{\tilde{B}} (\tilde{a}_j) - n_j\right)  
-  \frac{1}{8} \mu \delta (1-\delta) \sigma^2  + \frac{3}{2} \epsilon K {\frac{L^{3/2}}{\mu^{1/2}}}, \nonumber \\
& \hspace{0.40cm} \underset{(\text{Lemma~\ref{fbound}})}{\leq} \hspace{0.40cm} 
\max_j f\left(\mathcal{R}^f_{\tilde{B}} (\tilde{a}_j)\right)  
-  \frac{1}{8} \mu \delta (1-\delta) \sigma^2  + \frac{9}{2} \epsilon K {\frac{L^{3/2}}{\mu^{1/2}}}, \label{eqmp} 
\end{align} 
where $\tilde{a}_j$ is the perturbed column of $M$ corresponding to $w_j$ (that is, $\tilde{a}_j = w_j + n_j$).   
The second inequality follows from Lemma~\ref{fbound} since, by convexity of $||.||_2$ and by Lemma~\ref{lem1}, we have   
\[
\left\|\mathcal{R}^f_{\tilde{B}} (W) h_i \right\|_2 
\leq 
\max_i \left\|\mathcal{R}^f_{\tilde{B}} (w_i)\right\|_2 
\leq 
\sqrt{\frac{L}{\mu}} K. 
\] 
The third inequality is strict since, by strong convexity of $f$, the maximum is attained at a vertex with $x(\ell) = 1-\delta$ for some $1 \leq \ell \leq k$ at optimality while we assumed $h_i(\ell) < 1 - \delta$ for $1 \leq \ell \leq k$. 
The last inequality follows from Lemma~\ref{fbound} since 
\[
\left\|\mathcal{R}^f_{\tilde{B}} (\tilde{a}_j)\right\|_2 
\leq 
\sqrt{\frac{L}{\mu}} ||\tilde{a}_j||_2 
\leq  
\sqrt{\frac{L}{\mu}} (K+\epsilon) \leq 2 \sqrt{\frac{L}{\mu}} K . 
\] 
As $\delta \leq \frac{1}{2}$, we have
\begin{equation} \nonumber 
 \frac{1}{8} \mu \delta (1-\delta) \sigma^2 
\geq   \frac{1}{16} \mu \sigma^2 \delta 
= \frac{1}{16} \mu \sigma^2 \left( \frac{72  \epsilon K L^{3/2}}{ \sigma^2 \mu^{3/2}} \right) 
=  \frac{9}{2} \epsilon K {\frac{L^{3/2}}{\mu^{3/2}}}. 
\end{equation}
Combining this inequality with Equation~\eqref{eqmp}, we obtain 
$f\left(\mathcal{R}^f_{\tilde{B}}(\tilde{m}_i)\right) 
< 
\max_j f\left(\mathcal{R}^f_{\tilde{B}}(\tilde{a}_j)\right)$, 
a contradiction since $\tilde{m}_i$ should maximize $f\left(\mathcal{R}^f_{\tilde{B}}(.)\right)$ among the columns of $\tilde{M}$ and the $\tilde{a}_j$'s are among the columns of $\tilde{M}$.

To prove Equation~\eqref{errbound}, we use Equation~\eqref{eqmdel2} and observe that 
\begin{equation} \nonumber
m_i = (1-\delta') w_{\ell} + \sum_{k\neq {\ell}} \beta_k w_{k} \quad \text{ for some ${\ell}$ and } 1-\delta' \geq 1-\delta,  
\end{equation} 
so that $\sum_{k\neq {\ell}} \beta_k \leq \delta' \leq \delta$. Therefore, 
\begin{align*}
\left\|m_i - w_{\ell}\right\|_2 
& = \left\| - \delta' w_{\ell} + \sum_{k\neq {\ell}} \beta_k w_k \right\|_2  
 \leq 2 \delta' \max_j ||w_j||_2 \leq 2 \delta' K \leq 2 K \delta, 
\end{align*} 
 which gives
\[
||\tilde{m}_i - w_{\ell}||_2 \leq ||(\tilde{m}_i-m_i) + (m_i - w_{\ell})||_2 \leq \epsilon + 2 K \delta, 
\] 
for some $1 \leq {\ell} \leq k$.  
\end{proof}

We can now prove robustness of \ref{snpa} when $W$ has full column rank. 

\begin{theorem} \label{threc} 
Let $\tilde{M} = W H + N \in \mathbb{R}^{m \times n}$ satisfy Assumption~\ref{asssep} with $m \geq r$, and let $f$ satisfy Assumption~\ref{fass1} with strong convexity parameter $\mu$ and its gradient with Lipschitz constant $L$. 
Let us denote $K = K(W)$, $\sigma =\sigma(W)$, and let $||N||_{1,2} \leq \epsilon$ with  
\begin{equation} \label{sepnmfbound} 
\epsilon < \min \left(\frac{\alpha \mu}{4 L}, \frac{\sigma}{2 \sqrt{2r}} \right) \left( 1 + 144   \frac{K^2}{\sigma^2}  \frac{L^{3/2}}{\mu^{3/2}} \right)^{-1} . 
\end{equation} 
Let also $\mathcal{K}$ be the index set of cardinality $r$ extracted by Algorithm~\ref{snpa}. Then there exists a permutation $\pi$ of $\{1,2,\dots,r\}$ such that 
\begin{align*}
\max_{1 \leq j \leq r} ||\tilde{m}_{\mathcal{K}(j)} - w_{\pi(j)} ||_2 
& \leq \bar{\epsilon} 
= \epsilon \left(1+ 144 \frac{K^2}{\sigma^2} \frac{L^{3/2}}{\mu^{3/2}} \right) .  
\end{align*}
\end{theorem}
\begin{proof} 
The result follows using Theorem~\ref{Th2} inductively with 
\[
C = \left( 1 + 144 \frac{K^2}{\sigma^2} \frac{L^{3/2}}{\mu^{3/2}} \right). 
\]
The matrix  $B$ in Theorem~\ref{Th2} corresponds to the columns of $W$ extracted so far by \ref{snpa} (note that at the first step $B$ is the empty matrix) while the columns of $A$ correspond to the columns of $W$ not extracted yet. By Theorem~\ref{Th2}, if the columns of $B$ are at distance at most $\bar{\epsilon} = C \epsilon$ of some columns of $W$, then the next extracted column will be a column of the matrix $A$ and be at distance at most $\bar{\epsilon}$ of another column of $W$. The results therefore follows by induction. 
  
Note that $\epsilon < \frac{\sigma}{2 C \sqrt{2r}}$ implies $\epsilon < \frac{ \sigma^2 \mu^{3/2}}{144 K L^{3/2}}$ since $\sigma \leq K$ and 
\[
\epsilon 
< \frac{\sigma}{C} 
= \frac{\sigma}{1 + 144 \frac{K^2}{\sigma^2} \frac{L^{3/2}}{\mu^{3/2}}} 
\leq \frac{ \sigma^3 \mu^{3/2}}{144 K^2 L^{3/2}}  
\leq \frac{ \sigma^2 \mu^{3/2}}{144 K L^{3/2}}. 
\]
\end{proof}

Finally, \ref{snpa} with $f(.) = ||.||_2^2$ satisfies the same error bound as \ref{spa} when $W$ has full column rank: 
\begin{theorem}  \label{th4}
Let $\tilde{M}$ be a near-separable matrix satisfying  Assumption~\ref{asssep} where $W$ has full column rank. 
\mbox{If $\epsilon \leq \mathcal{O} \left( \,  \frac{  \sigma_{\min}(W)  }{\sqrt{r} \kappa^2(W)} \right)$}, 
then Algorithm~\ref{snpa} with $f(.) = ||.||_2^2$ identifies all the columns of $W$ up to error $\mathcal{O} \left( \epsilon \, \kappa^2(W) \right)$. 
\end{theorem}
\begin{proof}
This follows directly from $K(W) \le \sigma_{\max}(W)$, $\alpha(W) \geq \sigma(W)$ (Lemma~\ref{alsi}), Theorem~\ref{Th2} and the fact that $\mu = L = 2$ for $f(x) = ||x||_2^2$.
\end{proof}

\subsection{Generalization to Column-Rank-Deficient $W$}  \label{genmat}

SNPA can be applied to a broader class of near-separable matrices: 
in fact, the assumption that $W$ must be full column rank in Theorem~\ref{th4} is \emph{not necessary} for \ref{snpa} to be robust to noise. 
We now define the parameter $\beta(W)$, which will replace $\sigma(W)$ in the robustness analysis of \ref{snpa}. 
\paragraph{Definition (Parameter $\beta$)} 
Given $W \in \mathbb{R}^{m \times r}$ and $f$ satisfying Assumption~\ref{fass1}, we define 
\[
\nu_{\beta}(W) =  \min_{j} \left\|\mathcal{R}_{W(:,\mathcal{J})}^f(w_j)\right\|_2 \quad \text{ with } \mathcal{J} = \{ 1,2,\dots,r \} \backslash \{j\},   
\]
\[
\gamma_{\beta}(W) = \min_{i \neq j} \left\|\mathcal{R}^f_{W(:,\mathcal{J})}(w_i) - \mathcal{R}^f_{W(:,\mathcal{J})}(w_j) \right\|_2  \quad \text{ with } \mathcal{J} \subseteq \{ 1,2,\dots,r \} \backslash \{i,j\},   
\] 
and 
\[
\beta(W) = \min \left( \nu_{\beta}(W)  , \frac{1}{\sqrt{2}}\gamma_{\beta}(W) \right). 
\]

\noindent The quantity $\beta(W)$ is the minimum between 
\begin{itemize} 

\item  the norms of the residuals of the projections of the columns of $W$ onto the convex hull of the other columns of $W$, and 

\item  the distances between these residuals. 

\end{itemize} 
For example, if the columns of $W$ are the vertices of a triangle in the plane ($W \in \mathbb{R}^{2 \times 3})$, \ref{spa} can only extract two columns 
(the residual will be equal to zero after two steps because $\rank(W) = 2$) 
while, in most cases, $\beta(W) > 0$ and \ref{snpa} is able to identify correctly the three vertices even in the presence of noise. 
In particular, $\beta(W)$ is larger than $\sigma(W)$ hence is positive for matrices with full column rank:  
\begin{lemma} \label{sba} 
For any $W \in \mathbb{R}^{m \times r}$,   
$\beta(W) \geq \sigma(W)$. 
\end{lemma}
\begin{proof}
 This follows directly from Lemma~\ref{sig}. 
\end{proof}

However, the condition $\beta(W) > 0$ is not necessarily satisfied for any set of vertices $w_i$'s hence \ref{snpa} is not robust to noise for any matrix $W$ with $\alpha(W) > 0$. 
For example, in case of a triangle in the plane, $\beta(W) = 0$ if and only if the residuals of the projections of two columns of $W$ onto the segment joining the origin and the last column of $W$ are equal to one another (this requires that they are on the same side and at the same distant of that segment). It is the case for example for the following matrix 
\[
W = \left( \begin{array}{cccc} 
4 & 1  & 3  \\
0 & 1  & 1    \end{array} \right)  
\] 
with $\beta(W) = 0$ while $\alpha(W) > 0$, 
and any data point on the segment $[W(:,2), W(:,3)]$ could be extracted at the second step of \ref{snpa}. In fact, we have 
\[
\mathcal{R}^f_{W(:,1)}\Big( W(:,2) \Big) 
= \mathcal{R}^f_{W(:,1)}\Big( W(:,3) \Big)
=  
\left( \begin{array}{c} 
0   \\
1    \end{array} \right). 
\]

We can link $\beta(W)$ and $\alpha(W)$ as follows. 

\begin{lemma} \label{albet} 
For any $W \in \mathbb{R}^{m \times r}$,  
$\alpha(W) \geq \sqrt{\frac{\mu}{L}} \beta(W)$. 
\end{lemma}
\begin{proof}
Denoting $\mathcal{J} = \{1,2,\dots,r\} \backslash \{ j \}$, we have 
\begin{align*} 
\alpha(W)  
& = \min_{1 \leq j \leq r} \min_{x \in \Delta} ||w_j - W(:,\mathcal{J}) x||_2 \\ 
& \geq \sqrt{\frac{2}{L}} \min_{1 \leq j \leq r} \min_{x \in \Delta} f\left( w_j - W(:,\mathcal{J}) x \right) \\ 
& = \sqrt{\frac{2}{L}} \min_{1 \leq j \leq r} f\left( \mathcal{R}^f_{W(:,\mathcal{J})}(w_j) \right) \\
& \geq \sqrt{\frac{\mu}{L}} \min_{1 \leq j \leq r} \left\| \mathcal{R}^f_{W(:,\mathcal{J})}(w_j) \right\|_2 \geq \sqrt{\frac{\mu}{L}} \beta(W) .  
\end{align*} 
\end{proof}

In the following, we get rid of $\sigma(W)$ in the robustness analysis of \ref{snpa} to replace it with $\beta(W)$. 
In Theorems~\ref{Th2} and \ref{threc}, $\sigma(W)$ was used to lower bound $\omega\left( \mathcal{R}^f_{\tilde{B}}(A) \right)$; see in particular Equation~\eqref{epssi}. 
Hence we need to lower bound $\omega\left( \mathcal{R}^f_{\tilde{B}}(A) \right)$ using $\beta(W)$. 
The remaining of the proofs follows exactly the same steps as the proofs of Theorem~\ref{Th2} and \ref{threc}, and we do not repeat them here. 
\begin{theorem}  \label{mazette}
Let $\tilde{M} = WH + N \in \mathbb{R}^{m \times n}$ be a near-separable matrix satisfying Assumption~\ref{asssep}, and let $f$ satisfy Assumption~\ref{fass1} with strong convexity parameter $\mu$ and its gradient with Lipschitz constant $L$. 
Let us denote $K = K(W)$, $\beta =\beta(W)$, and let $||N||_{1,2} \leq \epsilon$ with  
\begin{equation} \label{epsbet} 
 \epsilon < \min \left( \frac{{\beta}^2 \mu^{3/2}}{144 K L^{3/2}}, \frac{\alpha \mu}{4 L C}, \frac{\beta^2 \mu}{128 K L C} \right), 
 \end{equation}
 with $C =  \left( 1 + 144 \frac{K^2}{\beta^2} \frac{L^{3/2}}{\mu^{3/2}} \right)$. 
Let also $\mathcal{K}$ be the index set of cardinality $r$ extracted by Algorithm~\ref{snpa}.
Then there exists a permutation $\pi$ of $\{1,2,\dots,r\}$ such that 
\begin{align*}
\max_{1 \leq j \leq r} ||\tilde{m}_{\mathcal{K}(j)} - w_{\pi(j)} ||_2 
& \leq \bar{\epsilon} 
= C \epsilon.  
\end{align*} 
\end{theorem} 
\begin{proof}
Using the first four assumptions of Theorem~\ref{Th2} (that is, all assumptions of Theorem~\ref{Th2} but the upper bound on $\epsilon$) and denoting $\beta = \beta(W)$, we show in Appendix~\ref{appB} that  
\[
 \omega\left( \mathcal{R}^f_{\tilde{B}}(A) \right) 
 \geq  \beta - 2 \sqrt{ \frac{6 K L \bar{\epsilon}}{\mu} } . 
\]
Therefore, if 
\[
2 \sqrt{ \frac{6 K L \bar{\epsilon}}{\mu} } \leq \frac{\beta}{2} \qquad  \iff \qquad 
\bar{\epsilon} = C \epsilon \leq \frac{\beta^2 \mu}{96 K L}, 
\]
then $\omega\left( \mathcal{R}^f_{\tilde{B}}(A) \right) \geq \frac{\beta}{2}$.  
Hence $\sigma(W)$ can be replaced with $\beta(W)$ in Equations~\eqref{epssi}, \eqref{eqmp} and in the following derivations, 
and Theorem~\ref{Th2} applies under the same conditions except for  
$\delta =  \frac{72  \epsilon K L^{3/2}}{\beta^2 \mu^{3/2}}$, and the condition on $\epsilon$ given by Equation~\eqref{epsbet}.  
\end{proof}

Let us denote $1 \leq \kappa_{\beta}(W) = \frac{K(W)}{ \beta(W)} \leq \kappa(W)$. 
\begin{theorem}  \label{maincor} 
Let $\tilde{M}$ be a near-separable matrix (see Assumption~\ref{asssep}) where $W$ satisfies $\beta(W) > 0$. 
\mbox{If $\epsilon \leq \mathcal{O} \left( \,  \frac{  \beta(W)  }{\kappa^3_{\beta}(W)} \right)$}, then Algorithm~\ref{snpa} with 
$f(.) = ||.||_2^2$ identifies all the columns of $W$ up to error $\mathcal{O} \left( \epsilon \, \kappa^3_{\beta}(W) \right)$. 
\end{theorem} 
\begin{proof}
This follows from Theorem~\ref{mazette}, Lemma~\ref{albet} and $\mu = L = 2$ for $f(.) = ||.||_2^2$.
\end{proof}

Note that the bounds are cubic in $\kappa_{\beta}(W)$ while they were quadratic in $\kappa(W)$. 
The reason is that the lower bound on $\omega\left( \mathcal{R}^f_{\tilde{B}}(A) \right)$ based on $\beta(W)$ is of the form $\beta(W) - \mathcal{O}(\sqrt{\bar{\epsilon}})$, 
while the one based on $\sigma(W)$ was of the form $\sigma(W) - \mathcal{O}(\sqrt{r} \bar{\epsilon})$. However the dependence on $\sqrt{r}$ has disappeared.

\begin{remark} \label{betaprime}
Because \ref{snpa} extracts the columns of $W$ in a specific order, the parameter $\beta(W) $ could be replaced by the following larger parameter. Assume without loss of generality that the columns of $W$ are ordered in such a way that the $k$th column of $W$ is extracted at the $k$th step of \ref{snpa}. Then, $\beta(W)$ in the robustness analysis of \ref{snpa} can be replaced with the larger 
\begin{align*}
\beta'(W) & = 
\Big( \min_{i < k} \left\| \mathcal{R}_{W(:,1:i)}^f(w_k) \right\|_2 , 
\frac{1}{\sqrt{2}}\min_{1\leq i \leq r-2, i < k \neq l} \left\| \mathcal{R}_{W(:,1:i)}^f(w_k) - \mathcal{R}_{W(:,1:i)}^f(w_l) \right\|_2 \Big). 
\end{align*} 
It is also interesting to notice that $\beta(W)$ could be equal to zero for some function $f$ while being positive for others. 
Hence, ideally, the function $f$ should be chosen such that $\beta(W)$ is maximized. Note however that $W$ is unknown and $\beta(W)$ is  expensive to compute (although $\beta'(W)$ could be used instead) so the problem seems rather challenging. This is a topic for further research.  
\end{remark}

\subsection{Improvements using Post-Processing, Pre-Conditioning and Outlier Detection} \label{improve}

It is possible to improve the performance of \ref{snpa}, the same way it was done for SPA: 

\begin{itemize}

\item A first possibility is to use \emph{post-processing} \cite[Alg.~4]{Ar13}. Let $\mathcal{K}$ be the set of indices extracted by \ref{snpa}, 
and denote $\mathcal{K}(k)$ the index extracted at step $k$. For $k = 1, 2, \dots r$, the post-processing 
\begin{itemize}
\item Projects each column of the data matrix onto the convex hull of $M(:,\mathcal{K} \backslash \{\mathcal{K}(k)\})$. 
\item Identifies the column of the corresponding projected matrix maximizing $f(.)$ (say the $k'$th), 
\item Updates $\mathcal{K} \, = \, \mathcal{K} \backslash \{\mathcal{K}(k)\} \cup \{k'\}$. 
\end{itemize} 
This allows to improve the bound on the error of Theorem~\ref{th1} from $\mathcal{O} \left( \epsilon \, \kappa^2(W) \right)$ to $\mathcal{O} \left( \epsilon \, \kappa(W) \right)$.

\item A second possibility is to \emph{pre-condition} the input near-separable matrix \cite{GV13}, making the condition number of $W$ constant, while multiplying the error by a factor of at most $\sigma_{\min}^{-1}(W)$. 
This allows to improve the bound on the noise of Theorem~\ref{th1} from $\epsilon \leq \mathcal{O} \left( \,  \frac{  \sigma_{\min}(W)  }{\sqrt{r} \kappa^2(W)} \right)$ to $\epsilon \leq \mathcal{O} \left( \,  \frac{  \sigma_{\min}(W)  }{r \sqrt{r}} \right)$ and the bound on the error  from $\mathcal{O} \left( \epsilon \, \kappa^2(W) \right)$ to $\mathcal{O} \left( \epsilon \, \kappa(W) \right)$. 

\item A third possibility for improvement is to deal with outliers. They will be identified by \ref{snpa} along with the columns of $W$, 
and can be discarded in a second step by computing the optimal weights needed to reconstruct all columns of the input matrix with the extracted columns~\cite{ESV12, GV12, GL13}.   

\end{itemize}

We do not focus in this paper on these improvements as they are straightforward applications of existing techniques. 
Our focus in Section~\ref{ne} is rather to show the better performance of \ref{snpa} compared to the original \ref{spa}.

\subsection{Choice of the Function $f$} \label{choicef}

According to our theoretical analysis (see Theorems~\ref{threc} and \ref{mazette}), the best possible case for the function $f$ is to have $\mu = L$ in which case $f(x) = ||x||_2^2$. However, our analysis considers a worst-case scenario and, in some cases, it might be beneficial to use other functions $f$. For example, if the noise is sparse (that is, only a few entries of the data matrix are perturbed), it was shown that it is better to use $\ell_p$-norms with $1 < p < 2$ (that is, use $f(x) = ||x||_p^p$); see the discussion in \cite[Section~4]{GV12}, and \cite{AC11} for more numerical experiments. 
More generally, it would be particularly interesting to analyze good choices for the function $f$ depending on the noise model. Note also that the assumptions on the function $f$ can be relaxed to the condition that the gradient of $f$ is continuously differentiable and that $f$ is locally strongly convex, hence our result for example applies to $\ell_p$ norms for $1 < p < +\infty$; see \cite[Remark~3]{GV12}.

\section{Numerical Experiments} \label{ne} 

In this section, we compare the following algorithms

\begin{itemize}

\item \textbf{SPA}: the successive projection algorithm; see Algorithm~\ref{spa}. 

\item \textbf{SNPA}: the successive nonnegative projection algorithm; see Algorithm~\ref{snpa} with $f(x) = ||x||_2^2$. 

\item \textbf{XRAY}: recursive algorithm similar to SNPA~\cite{KSK12}. It extracts columns recursively, and projects the data points onto the convex cone generated by the columns extracted so far. (We use in this paper the variant referred to as $max$.)

\end{itemize} 
Our goal is two-fold: 
\begin{enumerate}

\item Illustrate our theoretical result, namely that \ref{snpa} applies to a broader class of nonnegative matrices and is more robust than \ref{spa}; see Section~\ref{synt}. 

\item Show that \ref{snpa} can be used successfully on real-world hyperspectral images; see  Section~\ref{hsi} where the popular Urban data set is used for comparison. 

\end{enumerate}
We also show that \ref{snpa} is more robust to noise than XRAY (in some cases significantly). 

The Matlab code is available at \url{https://sites.google.com/site/nicolasgillis/}.  All tests are preformed using Matlab on a laptop Intel CORE i5-3210M CPU @2.5GHz 2.5GHz 6Go RAM. 

\begin{remark}[Comparison with Standard NMF Algorithms]
We do not compare the near-separable NMF algorithms to standard NMF algorithms whose goal is to solve 
\begin{equation} \label{nmfopt}
\min_{U\geq 0, V \geq 0} ||M-UV||_F^2 . 
\end{equation}
In fact, we believe these two classes of algorithms, although closely related, are difficult to compare. 
In particular, the solution of \eqref{nmfopt} is in general non-unique\footnote{A solution $(U',V')$ is considered to be different from $(U,V)$ if it cannot be obtained by permutation and scaling of the columns of $U$ and rows of $V$; see [15] for more details on non-uniqueness issues for NMF.}, even for a separable matrix $M = W[I_r, H']\Pi$. 
This is case for example if the support (the set of non-zero entries) of a column of $W$ contains the support of another column \cite[Remark~7]{G12}; in particular, most hyperspectral images have a dense endmember matrix $W$ hence have a non-unique NMF decomposition. (More generally, it will be the case if and only if the NMF of $W$ is non-unique.) 
Therefore, without regularization, standard NMF algorithms will in general fail to identify the matrix $W$ from a near-separable matrix $\tilde{M} = W[I_r, H']\Pi + N$. However, it is likely they will generate a solution with smaller error $||\tilde{M}-UV||_F^2$ than near-separable NMF algorithms because they have more degrees of freedom. 

It is interesting to note that near-separable NMF algorithms can be used as good initialization strategies for standard NMF algorithms (which usually require some initial guess for $U$ and $V$); see the discussion in \cite{G14} and the references therein.  
\end{remark}

\subsection{Synthetic Data Sets} \label{synt}

For well-conditioned matrices, $\beta(W)$ and $\sigma(W)$ are close to one another 
 (in fact, $\beta(I_r) = \sigma(I_r) = 1$), in which case we observed that \ref{spa} and \ref{snpa} provide very similar results (in most cases, they extract the same index set). Therefore, in this section, 
 our focus will be on 
\begin{itemize}

\item[(i)] near-separable matrices for which $W$ does not have full column rank (\mbox{$\rank(W) < r$})  to illustrate the fact that \ref{snpa} applies to a broader class of nonnegative matrices.  
We will refer to this case as the `rank-deficient' case; see Section~\ref{rankdefmat}. 

\item[(ii)] ill-conditioned matrices to illustrate the fact that \ref{snpa} is more tolerant to noise than \ref{spa}; in fact, our analysis suggests it is the case when $\sigma(W) \ll \beta(W)$.  
We will refer to this case as the `ill-conditioned' case; see Section~\ref{illcondmat}. 

\end{itemize}   
In both the rank-deficient and ill-conditioned cases, we take $r = 20$ and generate the matrices $H$ and $N$ (to obtain near-separable matrices $\tilde{M} = WH + N = W[I_r, H']  + N$) in two different ways (as in \cite{GV12}):  
\begin{enumerate}

\item \emph{Dirichlet.} $H = [I_r, I_r, H']$ so that each column of $W$ is repeated twice while $H'$ has 200 columns (hence $n = 240$) generated at random following a Dirichlet distribution with parameters drawn uniformly at random in the interval $[0,1]$. 
The repetition of the columns of the identity matrix allows to check whether algorithms are sensitive to duplicated columns of $W$ (some near-separable NMF algorithms are, e.g., \cite{EMO12, ESV12, BRRT12}).   
Each entry of the noise $N$ is drawn following a normal distribution $\mathcal{N}(0,1)$ and then multiplied by the parameter~$\delta$. 

\item \emph{Middle Points.} $H = [I_r, H']$ where each column of $H'$ has exactly two nonzero entries equal to 0.5 (hence $n = r+\binom{r}{2} = 210$). The columns of $W$ are not perturbed (that is, $N(:,1:r) = 0$) while the remaining ones are perturbed towards the outside of the convex hull of the columns of $W$, that is, $N(:,j) = \delta ( M(:,j) - \bar{w})$  for all $j$ 
where $\bar{w} = \frac{1}{r} \sum_{k=1}^r w_k$ and $\delta$ is a parameter. 

\end{enumerate}

\subsubsection{Rank-Deficient Case} \label{rankdefmat}

To generate near-separable matrices with $\rank(W) < r$, we take $m = 10$ (since $r = 20$, these matrices cannot have full column rank) 
while each entry of the matrix $W \in[0,1]^{m \times r}$ is drawn uniformly at random in the interval [0,1]. Moreover, we check that each column of $W$ is not too close to the convex cone generated by the other columns of $W$ by making sure that for all $j$ 
\[
\min_{x \geq 0} ||W(:,j) - W(:,\mathcal{J})x||_2 \geq 0.01 ||W(:,j)||_2 \quad \text{ with } \mathcal{J} = \{1,2,\dots,r\} \backslash \{j\}.  
\]
If this condition is not met (which occurs very rarely), we generate another matrix until it is.  
Hence, we have that $\rank(W) = \rank(M) = 10$ while $M = WH$ is 20-separable with \mbox{$\alpha(W) > 0$}. 
(By a slight abuse of language, we refer to this case as the `rank-deficient case' although the matrix $M$ actually has full (row) rank--to be more precise, we should refer to it as  the `column-rank-deficient case'.)  
It is interesting to point out that, in average, $\beta'(W) \approx 0.2$ (see Remark~\ref{betaprime} for the definition of $\beta'(W)$ which can replace $\beta(W)$ in the analysis of \ref{snpa}) while $\kappa_{\beta'}(W) = \frac{\max_j ||W(:,j)||_2}{\beta'(W)} \approx 10$. 

For hundred different values of the noise parameter $\delta$ (using \texttt{logspace(-3,0,100)}), 
we generate 25 matrices of each type: Figure~\ref{xp11} (resp.\@ Figure~\ref{xp12}) displays the fraction of columns of $W$ correctly identified by the different algorithms for the experiment `Dirichlet' (resp.\@ `Middle points'). 
\begin{figure*}[ht!]
\begin{center}
\includegraphics[width=14cm]{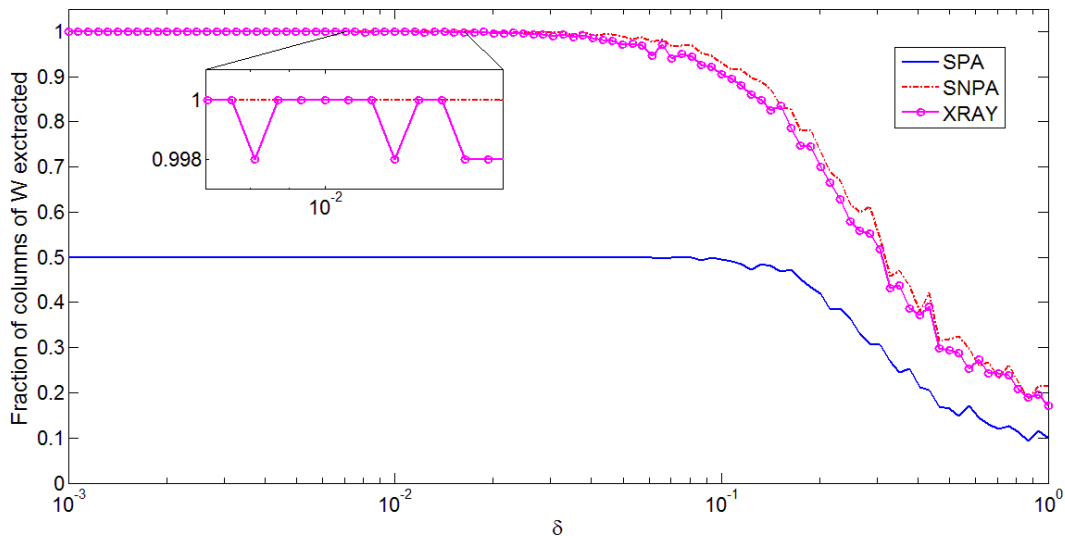}
\caption{Comparison of the different near-separable NMF algorithms on rank-deficient data sets (`Dirichlet' type).}  
\label{xp11}
\end{center}
\end{figure*} 
\begin{figure*}[ht!]
\begin{center}
\includegraphics[width=14cm]{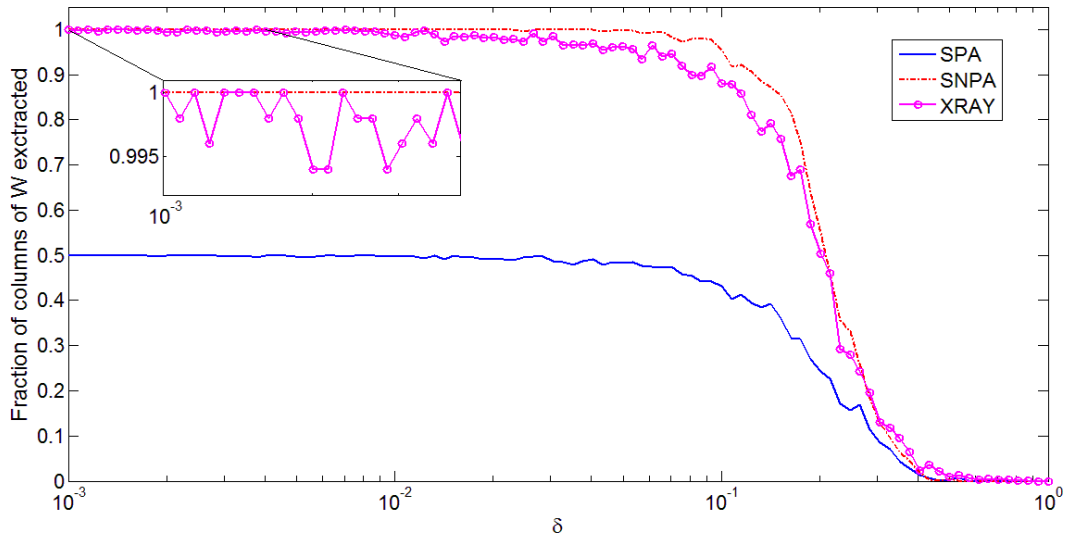}
\caption{Comparison of the different near-separable NMF algorithms on rank-deficient data sets  (`Middle points' type).}  
\label{xp12}
\end{center}
\end{figure*}

Table~\ref{tabrankdefdir} and Table~\ref{tabankdefmidpoints} give the robustness and the average running time for both experiments.  
\begin{table}[ht!] 
\begin{center} 
\caption{Robustness for the rank-deficient `Dirichlet' experiment, that is, largest value of $\delta$ for which all (resp.\@ 95\% of the) columns of $W$ are correctly identified, and average running time in seconds of the different near-separable NMF algorithms.} 
\begin{tabular}{|c|ccc|}
\hline
			&     \ref{spa}  & \ref{snpa} & XRAY   \\ \hline
Robustness (100\%) & 0  & \textbf{1.7}$^*$\textbf{10}$^{\textbf{\text{-2}}}$ & $7.6^*$10$^{-3}$  \\  
Robustness (95\%) &  0  & \textbf{8.9}$^*$\textbf{10}$^{\textbf{\text{-2}}}$   & {6.1$^*$10$^{{\text{-2}}}$}   \\  \hline 
Time (s.) &    $<$ 0.01 & 7.67 & 1.16   \\ 
\hline
\end{tabular} 
\label{tabrankdefdir}
\end{center}
\end{table}
\begin{table}[ht!] 
\begin{center} 
\caption{Robustness and average running time for the rank-deficient `Middle Points' experiment.} 
\begin{tabular}{|c|ccc|}
\hline
			&     \ref{spa}  & \ref{snpa} & XRAY   \\ \hline
Robustness (100\%) & 0  & \textbf{2.3}$^*$\textbf{10}$^{\textbf{\text{-2}}}$ & $10^{-3}$  \\  
Robustness (95\%) &  0  &  \textbf{10}$^{\textbf{\text{-1}}}$ &  {5.5}$^*${10}$^{{\text{-2}}}$ \\  \hline 
Time (s.) &    $<$ 0.01 & 7.83 & 1.05   \\ 
\hline
\end{tabular} 
\label{tabankdefmidpoints}
\end{center}
\end{table}

\ref{spa} is significantly faster than XRAY and \ref{snpa} since the projection at each step simply amounts to a matrix-vector product 
while, for \ref{snpa} and XRAY, the projection requires solving $n$ linearly constrained least squares problems in $r$ variables. 
XRAY is faster than \ref{snpa} because the projections are simpler to compute. (Also, a different algorithm was used: XRAY requires to solve least squares problems in the nonnegative orthant, which is solved with an efficient coordinate descent method from~\cite{GG12}.)

For both experiments (`Dirichlet' and `Middle points'), \ref{spa} cannot extract more than $10$ columns. In fact, the input matrix has only 10 rows hence the residual computed by SPA is equal to zero after 10 steps. 
\ref{snpa} is able to identify correctly the 20 columns of $W$, and it does for larger noise levels than XRAY. 
In particular, for the `Middle points' experiment, XRAY performs rather poorly (in terms of robustness) because it does not deal very well with data points on the faces of the convex hull of the columns of $W$ (in this case, on the middle of the segment joining two vertices); see the discussion in \cite{KSK12}. For example, for $\delta = 0.1$, \ref{snpa} identifies more than 95\% of the columns of $W$, while XRAY identifies less than 90\%.

\subsubsection{Ill-Conditioned Case} \label{illcondmat}
 
In this section, we perform an experiment very similar to the third and fourth experiment in \cite{GV12} to assess the robustness to noise of the different algorithms on ill-conditioned matrices. 
We take $m = 20$, while each entry of the matrix $W \in[0,1]^{m \times r}$ is drawn uniformly at random in the interval [0,1] (as in the rank-deficient case). 
Then the compact singular value decomposition $(U,S,V^T)$ of $W$ is computed (using the function \texttt{svds(M,r)} of Matlab), 
and $W$ is replaced with $U \Sigma V^T$ where $\Sigma$ is a diagonal matrix with $\Sigma(i,i) = \alpha^{i-1}$ ($1 \leq i \leq r$) 
where $\alpha^{r-1} = 1000$ so that $\kappa(W) = 1000$. Finally, to obtain a nonnegative matrix $W$, we replace $W$ with $\max(W,0)$ (this step is necessary because XRAY only applies to nonnegative input matrices).  
Note that this changes the conditioning, with the average of $\kappa(W)$ being equal to $5000$. 
It is interesting to point out that for these matrices ${\beta'}(W)$ is usually much smaller than $\sigma(W)$: 
the order of magnitudes are $\beta'(W) \approx 10^{-2} \ll \sigma_{\min(W)} \approx 10^{-3}$ while $\kappa_{\beta'}(W) \approx 10 \ll \kappa(W) \approx 10^3$. 


For hundred different values of the noise parameter $\delta$ (using \texttt{logspace(-4,-0.5,100)}), we generate 25 matrices of the two types: Figure~\ref{xp1} (resp.\@ Figure~\ref{xp2}) displays the fraction of columns of $W$ correctly identified by the different algorithms for the experiment `Dirichlet' (resp.\@ `Middle points'). 
\begin{figure*}[ht!]
\begin{center}
\includegraphics[width=14cm]{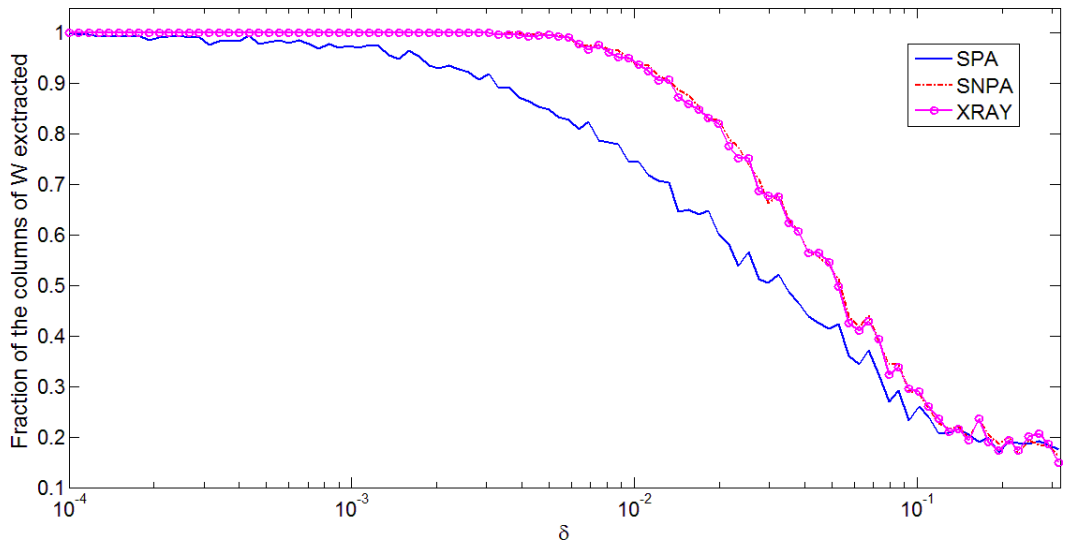}
\caption{Comparison of the different near-separable NMF algorithms on ill-conditioned data sets (`Dirichlet' type).}  
\label{xp1}
\end{center}
\end{figure*} 
\begin{figure*}[ht!]
\begin{center}
\includegraphics[width=14cm]{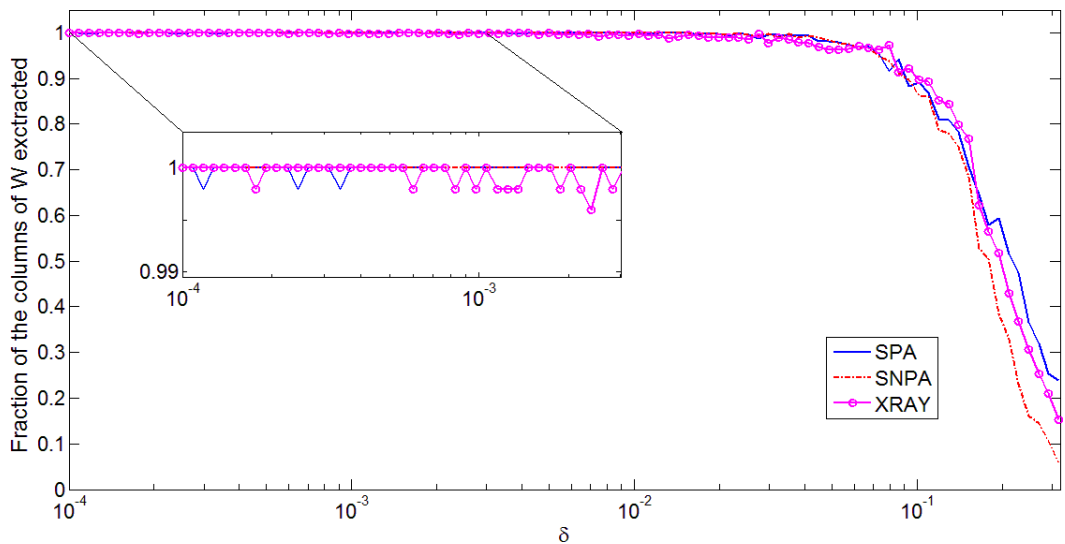}
\caption{Comparison of the different near-separable NMF algorithms on ill-conditioned data sets  (`Middle points' type).}  
\label{xp2}
\end{center}
\end{figure*} 
Table~\ref{robtim} and Table~\ref{robtim2} give the robustness and the average running time for both experiments.  
\begin{table}[ht!] 
\begin{center}
\caption{Robustness and average running time for the ill-conditioned `Dirichlet' experiment.} 
\begin{tabular}{|c|ccc|}
\hline
			&     \ref{spa}  & \ref{snpa} & XRAY   \\ \hline
Robustness (100\%) & $10^{-4}$  & \textbf{3.1$^*$10$^{\textbf{\text{-3}}}$} & \textbf{3.1$^*$10$^{-3}$}  \\  
Robustness (95\%) & $1.44^*10^{-3}$  & \textbf{9.45$^*$10$^{\textbf{\text{-3}}}$}  &  \textbf{9.45$^*$10$^{\textbf{\text{-3}}}$}  \\  \hline 
Time (s.) &    $<$ 0.01 & 7.43 & 1.08    \\ 
\hline
\end{tabular} 
\label{robtim}
\end{center}
\end{table} 
\begin{table}[ht!] 
\begin{center} 
\caption{Robustness and average running time for the ill-conditioned `Middle points' experiment.} 
\begin{tabular}{|c|ccc|}
\hline
			&     \ref{spa}  & \ref{snpa} & XRAY   \\ \hline
Robustness (100\%) & $1.1^*10^{-4}$  & \textbf{1.6$^*$10$^{\textbf{\text{-2}}}$} & $1.6^*$10$^{-4}$  \\  
Robustness (95\%) & $7.4^*10^{-2}$  & 7.3$^*$10$^{-2}$  &  \textbf{8.2}$^*$\textbf{10}$^{\text{\textbf{-2}}}$  \\  \hline 
Time (s.) &    $<$ 0.01 & 9.22  & 1.32   \\ 
\hline
\end{tabular} 
\label{robtim2}
\end{center}
\end{table} 

For the same reasons as before, \ref{spa} is significantly faster than XRAY which is faster than \ref{snpa}. 

For both experiments (`Dirichlet' and `Middle points'), \ref{snpa} outperforms \ref{spa} in terms of robustness, as expected by our theoretical findings. In fact, \ref{snpa} is about ten (resp.\@ hundred) times more robust than \ref{spa} for the experiment `Dirichlet' (resp.\@ `Middle points'), that is, it identifies correctly all columns of $W$ for the noise parameter $\delta$ ten (resp.\@ hundred) times larger; see Tables~\ref{robtim} and~\ref{robtim2}.   
Moreover, for the `Dirichlet' experiment, \ref{snpa} identifies significantly more columns of $W$, even for larger noise levels (which fall outside the scope of our analysis); for example, for $\delta = 0.01$, \ref{snpa} identifies correctly about 95\% of the columns of $W$ while \ref{spa} identifies about 75\%. 


For the `Dirichlet' experiment, \ref{snpa} is as robust as XRAY while, for the `Middle points' experiment, 
it is significantly more robust (for the same reason as in the rank-deficient case) as it extracts correctly all columns of $W$ for $\delta$ hundred times larger. 
However, the three algorithms overall perform similarly on the `Middle points' experiment in the sense that the fraction of columns of $W$ correctly identified do not differ by more than about five percent for all $\delta \leq 0.1$.

\subsection{Real-World Hyperspectral Image}  \label{hsi}

In this section, we analyze the Urban data set\footnote{Available at \url{http://www.agc.army.mil/}.} with $m = 162$ and $n = 307 \times 307 = 94249$. It is mainly constituted of grass, trees, dirt, road and different roof and metallic surfaces; see Figure~\ref{urband}. 
\begin{figure}[ht!] 
\begin{center}
\begin{tabular}{c}
\includegraphics[height=6cm]{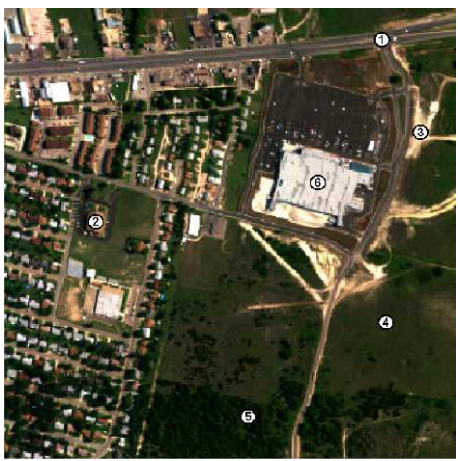} \\ 
\end{tabular}
\caption{Urban data set taken from an aircraft (army geospatial center) with  road surfaces~(1), roofs~1~(2), dirt~(3), grass~(4), trees~(5) and roofs~2~(6).}   
\label{urband}
\end{center}
\end{figure}

We run the near-separable NMF algorithms to  extract $r = 8$ endmembers. As mentioned in Section~\ref{assdef}, Assumption~\ref{asssep} is naturally satisfied by hyperspectral images (up to permutation) hence no normalization of the data is necessary. 
On this data set, \ref{spa} took less than half a second to run, \ref{snpa} about one minute and XRAY about half a minute.   

Figure~\ref{urbanss} displays the extracted spectral signatures, 
\begin{figure*}[ht!] 
\begin{center}
\includegraphics[width=\textwidth]{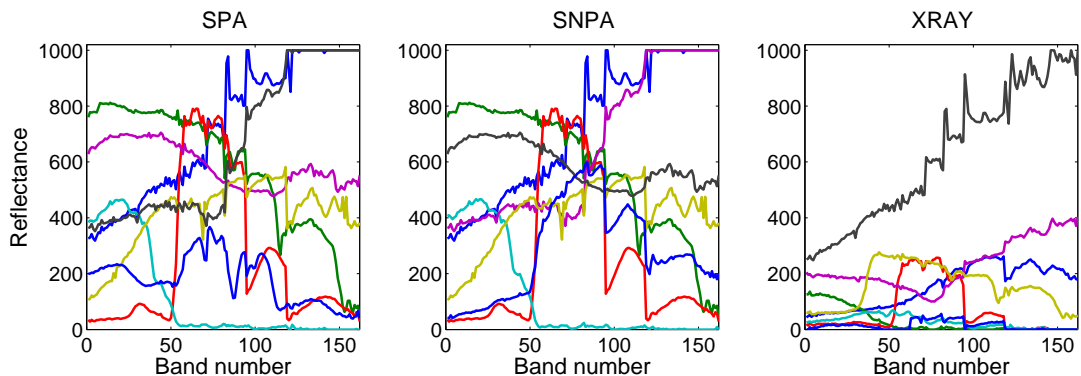} 
\caption{Spectral signatures of the extracted endmembers.}   
\label{urbanss}
\end{center}
\end{figure*} 
and Figure~\ref{urbanam} displays the corresponding abundance maps, that is, the rows of 
\[
H^* = \argmin_{H \geq 0} ||M-M(:\mathcal{K})H||_F 
\] 
where $\mathcal{K}$ is the set of extracted indices by a given algorithm.   
SPA and \ref{snpa} extract six common indices (out of the eight, the first one being different is the fourth). 
\begin{figure*}[ht!] 
\begin{center}
\includegraphics[width=\textwidth]{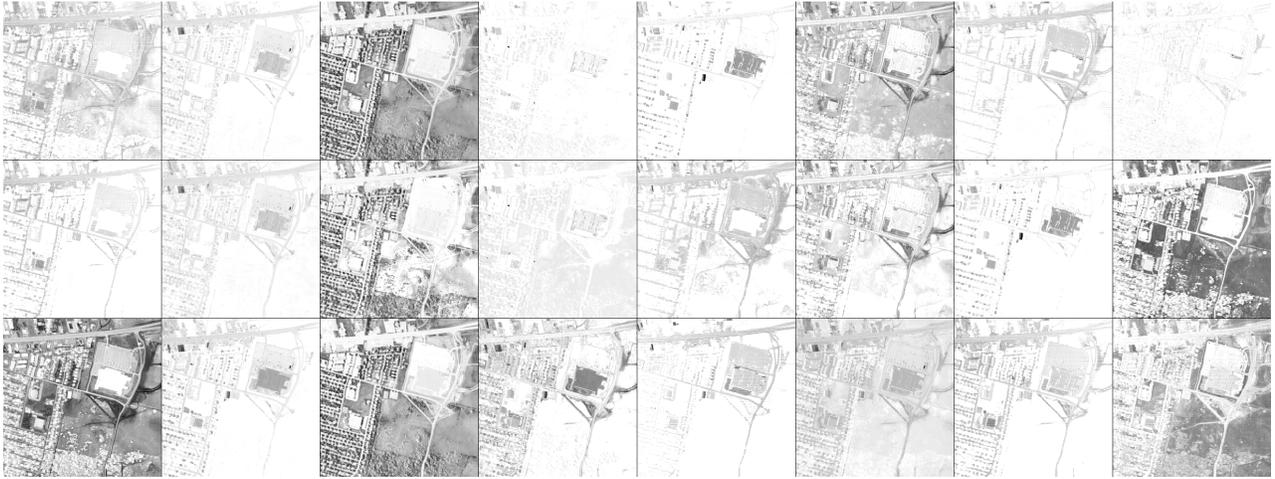} 
\caption{Abundances maps corresponding to the extracted indices. From top to bottom: \ref{spa}, \ref{snpa} and XRAY.} 
\label{urbanam}
\end{center}
\end{figure*}

SNPA performs better than both \ref{spa} and XRAY as it is the only algorithm able to distinguish the grass and trees (3rd and 8th extracted endmembers), while identifying the road surfaces (1st), the dirt (6th) and roof tops (7th). In particular, the relative error in percent, that is, 
\[
100 * \frac{\min_{H \geq 0} ||M-M(:,\mathcal{K})H||_F}{||M||_F} \in [0,100]
\]
for \ref{spa} is 9.45, for XRAY 6.82 and for \ref{snpa} 5.64. In other words, \ref{snpa} is able to identify eight columns of $M$ which can reconstruct $M$ better. (Note that, as opposed to SPA and SNPA that looks for endmembers with large norms, XRAY focuses on extracting extreme rays of the convex cone generated by the columns of $M$. Hence it is likely for XRAY to identify columns with smaller norms. This explains the different scaling of the extracted endmembers in Figure~\ref{urbanss}.) 

Further research includes the comparison of \ref{snpa} with other endmember extraction algorithms, and its incorporation in more sophisticated techniques, e.g., where pre-processing is used to remove outliers and noise, or where pure-pixel search algorithms (that is, near-separable NMF algorithms) are used as an initialization for more sophisticated (iterative) methods not relying on the pure-pixel assumption; see, e.g., \cite{CM11} where \ref{spa} is used.

\subsection{Document Data Sets}

Because, as for \ref{spa}, \ref{snpa} requires to normalize the input near-separable matrices not satisfying Assumption~\ref{asssep} 
(that is, near-separable matrices for which the columns of $H$ do not belong to $\Delta$), 
it may introduce distortion in the data set \cite{KSK12}. In particular, the normalization amplifies the noise of the columns of $M$ with small norm (see the discussion in \cite{GL13}).  

In document data set, the columns of the matrix $H$ are usually not assumed to belong to the unit simplex hence normalization is necessary for applying \ref{snpa}. Therefore, XRAY should be preferred and it has been observed that, for document data sets, \ref{snpa} and \ref{spa} perform similarly while XRAY performs better \cite{Kumar}; see also \cite{KSK12}.

\section{Conclusion and Further Research}

In this paper, we have proposed a new fast and robust recursive algorithm for near-separable NMF, which we referred to as the successive nonnegative projection algorithm (\ref{snpa}). 
Although computationally more expensive than the successive projection algorithm (\ref{spa}), \ref{snpa} can be used to solve large-scale problems, running in $\mathcal{O}(mnr)$ operations, while being more robust and applicable to a broader class of nonnegative matrices. 
In particular, \ref{snpa} seems to be a good alternative to \ref{spa} for real-world hyperspectral images.

There exists several algorithms robust for any near-separable matrix requiring only that $\alpha(W) > 0$ \cite{AGKM11, G12, GL13} which are therefore more general than \ref{snpa} which requires $\beta(W) > 0$. In fact, under Assumption~\ref{asssep}, $\alpha(W) > 0$ is a necessary condition for being able to identify the columns of $W$ among the columns of $\tilde{M}$. 
However, these algorithms are computationally much more expensive ($n$ linear programs in $\mathcal{O}(n)$ variables or a single linear program in $\mathcal{O}(n^2)$ variables have to be solved). Therefore, it would be an interesting direction for further research to develop, if possible, faster (recursive?) algorithms provably robust for any near-separable matrix $\tilde{M} = W[I_r, H'] + N$ with $\alpha(W) > 0$.

\section*{Acknowledgment} 

The author would like to thank Abhishek Kumar and Vikas Sindhwani (IBM T.J. Watson Research Center) for motivating him to study the robustness of algorithms based on nonnegative projections, for insightful discussions and for performing some numerical experiments on document data sets.  The author would also like to thank Wing-Kin Ma (The Chinese University of Hong Kong) for insightful discussions and for suggesting to analyze the noiseless case separately. 
The authors is grateful to the reviewers for their insightful comments which helped improve the paper significantly.

\appendix

\section{Fast Gradient Method for Least Squares on the Simplex}  \label{appA} 
Algorithm~\ref{fastgrad} is a fast gradient method to solve  
\begin{equation} \label{fcnls}
\min_{x \in \Delta^r} f(Ax - y) , 
\end{equation}
where $y \in \mathbb{R}^m$ and $A \in \mathbb{R}^{m \times r}$. To achieve an accuracy of $\epsilon$ in the objective function, 
the algorithm requires $\mathcal{O}\left( \frac{1}{\sqrt{\epsilon}}\right)$ iterations. In other words, the objective function converges to the optimal value at rate $\mathcal{O}\left( \frac{1}{k^2} \right)$ where $k$ is the iteration number. 

\renewcommand{\thealgorithm}{FGM} 
 \algsetup{indent=2em}
\begin{algorithm}[ht!]
\caption{Fast Gradient Method for Solving \eqref{fcnls}; see \cite[p.90]{Y04}} \label{fastgrad}
\begin{algorithmic}[1] 
\REQUIRE A point $y \in \mathbb{R}^m$, a matrix $A \in \mathbb{R}^{m \times r}$, a function $f$ whose gradient is Lipschitz continuous with constant $L_f$, and an initial guess $x \in \Delta^{r}$.  
\ENSURE An approximate solution $x \approx \argmin_{z \in \Delta} f(Az-y)$ so that $Ax \approx \mathcal{P}^f_A(y)$.  \medskip 
\STATE $\alpha_0 \in (0,1)$; $z = x$; $L = L_f \sigma_{\max}(A)^2$ .
\FOR{$k = 1 :$ maxiter} 
\STATE $x^{\dagger} = x$. \quad \emph{\% Keep the previous iterate in memory.}
\STATE $x = \mathcal{P}_{\Delta}\Big( z - \frac{1}{L} \nabla f(Az-y) \Big)$. \emph{\% $\mathcal{P}_{\Delta}$ is the projection on $\Delta$; see Appendix~\ref{remproj}.} 
\STATE $z = x + \beta_k \left(x - x^{\dagger}\right)$, \quad where $\beta_k =  \frac{\alpha_{k} (1-\alpha_{k})}{\alpha_{k}^2 + \alpha_{k+1}}$ 
with 
$\alpha_{k+1} \geq 0$ s.t. 
$\alpha_{k+1}^2 = (1-\alpha_{k+1}) \alpha_{k}^2$. 
\ENDFOR
\end{algorithmic} 
\end{algorithm}  

\begin{remark}
Note that the function $g(x) = f(Ax - b)$ is not necessarily strongly convex, even if $f$ is. This would require $A$ to be full column rank, which we we do not assume here. If it were the case, then even faster methods could be used although the convergence of Algorithm~\ref{fastgrad} becomes linear \cite{Y04}. 
\end{remark}

\begin{remark}[Stopping Condition]
For the numerical experiments in Section~\ref{ne}, we used \emph{maxiter = 500} and combined it with a stopping condition based on the evolution of the iterates; see the online code for more details. 
\end{remark}

\subsection{Projection on the Unit Simplex $\Delta$} \label{remproj}

In Algorithm~\ref{fastgrad}, the projection onto the unit simplex $\Delta$ needs to be computed, that is, given $y \in \mathbb{R}^r$, we have to compute 
\[
\mathcal{P}_{\Delta}(y) = x^* \, = \, \argmin_{x} \, \frac{1}{2} ||x-y||^2 \, \text{ such that } \, 
x \in \Delta. 
\]
Let us construct the Lagrangian dual corresponding to the sum-to-one constraint (since the problem above has a Slater point, there is no duality gap): 
\[
\max_{\mu \geq 0} \min_{x \geq 0} \; \frac{1}{2} ||x-y||^2 - \mu (1-e^T x), 
\]
where $e$ is the all-one vector and $\mu \geq 0$ the Lagrangian multiplier. 
For $\mu$ fixed, the optimal solution in $x$ is given by 
\[
x^* = \max(0,y - \mu e). 
\]
If the sum-to-one constraint is not active, that is, $\sum_i x^*_i < 1$ , we must have $\mu = 0$ hence $x^* = \max(0,y)$ (hence this happens if and only if $\max(0,y) \in \Delta$).  
Otherwise the value of $\mu$ can computed by solving the system $\sum_i x^*_i = 1$ and $x^* = \max(0,y - \mu e)$, equivalent to finding $\mu$ satisfying $\sum_{i=1}^n \max(0,y_i - \mu) = 1$ ($\mu$ can be found easily after having sorted the entries of~$y$).

\section{Lower Bound for $\omega\left( \mathcal{R}^f_{\tilde{B}} (A) \right)$ depending on $\beta\left([A,B]\right)$} \label{appB} 
In this appendix, we derive a lower bound on $\omega\left(\mathcal{R}^f_{\tilde{B}} (A) \right)$ based on $\beta([A,B])$. 

\begin{lemma} \label{betlem1}
Let $x \in \mathbb{R}^{m}$, $B$ and $\tilde{B} \in \mathbb{R}^{m \times s}$ be such that 
$||B-\tilde{B}||_{1,2} \leq \bar{\epsilon} \leq ||B||_{1,2}$, and $f$ satisfy Assumption~\ref{fass1}. 
Then 
\[
\left\| \mathcal{R}^f_{B}(x) - \mathcal{R}_{\tilde{B}}^f(x) \right\|_2^2 
\leq  12 \, \frac{L}{\mu}  \, \bar{\epsilon} \, ||B||_{1,2}. 
\]
\end{lemma} 
\begin{proof} 
Let us denote 
 $z = \mathcal{P}_{B}^f(x)$,  $\tilde{z} = \mathcal{P}^f_{\tilde{B}}(x)$, and 
$z^* = \mathcal{P}^f_{[B, \tilde{B}]}(x)$.   We have 
\begin{align*}
\left\| \mathcal{R}^f_{B}(x) - \mathcal{R}_{\tilde{B}}^f(x) \right\|_2 
& = 
\left\| (x - \mathcal{P}^f_{B}(x)) - ( x - \mathcal{P}^f_{\tilde{B}}(x)) \right\|_2 
 = \left\|  \mathcal{P}^f_{B}(x) - \mathcal{P}^f_{\tilde{B}}(x) \right\|_2 
=  \left\|  z - \tilde{z} \right\|_2. 
\end{align*}
Since $z^* = [B, \tilde{B}]w^*$ for some $w^* \in \Delta$, there exists $y=Bw$ and $\tilde{y} = \tilde{B} \tilde{w}$ with $w, \tilde{w}  \in \Delta$ such that
$||z^* - y ||_2 \leq \bar{\epsilon}$ and $||z^* - \tilde{y} ||_2 \leq \bar{\epsilon}$. In fact, it suffices to take 
$w = \tilde{w} = w^*(1$:$r) + w^*(r+1$:$2r)$ since  
\begin{align*}
||z^* - y ||_2 
& = ||[B, \tilde{B}] w^* - B w ||_2 \\
& = || B w^*(1:r) +  \tilde{B} w^*(r+1:2r)  - B w^*(1:r) - B w^*(r+1:2r)  ||_2  \\ 
& = || (\tilde{B} - B) w^*(r+1:2r)||_2 \leq ||B-\tilde{B}||_{1,2} \leq \bar{\epsilon}, 
\end{align*}
and similarly for $\tilde{y}$. 
 Therefore, there exists some $n, \tilde{n}$ such that $z^* = y + n = \tilde{y} + \tilde{n}$ with $||n||_2, ||\tilde{n}||_2 \leq \bar{\epsilon}$.  By Lemma~\ref{fbound} and the fact that $||y||_2 \leq ||B||_{1,2}$ and $||\tilde{y}||_2 \leq ||\tilde{B}||_{1,2} \leq ||B||_{1,2} + \bar{\epsilon} \leq 2  ||B||_{1,2}$, we have 
 \[
 f(z^*) = f(y + n) \geq f(y) -  ||B||_{1,2} L \bar{\epsilon}  
 \quad \text{ and } \quad 
 f(z^*) = f(\tilde{y} + \tilde{n}) \geq f(\tilde{y}) -  2 ||B||_{1,2} L \bar{\epsilon}. 
 \] 
Therefore, by definition of $z$ and $\tilde{z}$, 
\begin{align*}
f(z^*) 
& \geq 
\frac{1}{2} f(y) + \frac{1}{2} f(\tilde{y})  -  \frac{3}{2} ||B||_{1,2} L \bar{\epsilon}  \geq 
\frac{1}{2} f(z) + \frac{1}{2} f(\tilde{z})  -  \frac{3}{2} ||B||_{1,2} L \bar{\epsilon}. 
\end{align*}
Moreover, by definition of $z^*$ and strong convexity of $f$, we obtain 
\[ 
 f(z^*) 
  \leq f\left(\frac{1}{2}z + \frac{1}{2}\tilde{z}\right)  
  \leq \frac{1}{2} f(z) + \frac{1}{2}f(\tilde{z}) - \frac{\mu}{8} ||z-\tilde{z}||_2^2.  
\]
 Hence, combining the above two inequalities, $||z-\tilde{z}||_2^2 \leq 12 \frac{L}{\mu} ||B||_{1,2} \bar{\epsilon}$.  
\end{proof}

\begin{lemma} \label{betlem}
Let $x, y \in \mathbb{R}^{m}$, $B$ and $\tilde{B} \in \mathbb{R}^{m \times s}$ be such that 
$||B-\tilde{B}||_{1,2} \leq \bar{\epsilon} \leq ||{B}||_{1,2}$, and $f$ satisfy Assumption~\ref{fass1}. Then 
\[
\left\| \mathcal{R}_{\tilde{B}}^f(x) - \mathcal{R}_{\tilde{B}}^f(y) \right\|_2
\geq 
\left\| \mathcal{R}_B^f(x) - \mathcal{R}_B^f(y) \right\|_2 - 4 \sqrt{ \frac{3 KL}{\mu} \bar{\epsilon}}. 
\]
\end{lemma} 
\begin{proof}
This follows directly from Lemma~\ref{betlem1}:  
\begin{align*}
\Big\| \mathcal{R}_{\tilde{B}}^f(x)  - \mathcal{R}_{\tilde{B}}^f(y) \Big\|_2 
 & = \left\| \mathcal{R}_{\tilde{B}}^f(x) - \mathcal{R}^f_{{B}}(x) + \mathcal{R}^f_{{B}}(x) - \mathcal{R}^f_{\tilde{B}}(y) + \mathcal{R}^f_{{B}}(y) - \mathcal{R}^f_{{B}}(y) \right\|_2 \\
& \geq \left\|  \mathcal{R}^f_{{B}}(x) - \mathcal{R}^f_{{B}}(y) \right\|_2 - 2 \sqrt{ \frac{12 ||B||_{1,2} L \bar{\epsilon}}{\mu} } . 
\end{align*}
\end{proof}
 
\begin{lemma} 
Let $A \in \mathbb{R}^{m \times k}$, $B$ and $\tilde{B} \in \mathbb{R}^{m \times s}$ be such that $||B-\tilde{B}||_{1,2} \leq \bar{\epsilon}\leq ||B||_{1,2}$, and $f$ satisfy Assumption~\ref{fass1}. Then 
\[
\omega \left( \mathcal{R}^f_{\tilde{B}}(A) \right) 
\geq \beta([A, B]) - 2  \sqrt{ 6 \frac{L}{\mu} ||B||_{1,2}   \bar{\epsilon}} . 
\]
\end{lemma} 
\begin{proof}
This follows directly from Lemmas~\ref{betlem1} and \ref{betlem}. In fact, for all $i$, 
\begin{align*}
\beta([A, B]) & - \left\| \mathcal{R}_{\tilde{B}}^f(a_i) \right\|_2
 \leq 
\left\| \mathcal{R}^f_{B}(a_i)  \right\|_2 - \left\| \mathcal{R}_{\tilde{B}}^f(a_i) \right\|_2  \leq 
\left\| \mathcal{R}^f_{B}(a_i) - \mathcal{R}_{\tilde{B}}^f(a_i) \right\|_2 
 \leq  \sqrt{12 \, \frac{L}{\mu}  \, \bar{\epsilon} \, ||B||_{1,2}}, 
\end{align*}
while, for all $i, j$,  
\begin{align*}
\frac{1}{\sqrt{2}} \Big\|  \mathcal{R}_{\tilde{B}}^f(a_i)  - \mathcal{R}_{\tilde{B}}^f(a_j) \Big\|_2 
 & \geq 
\frac{1}{\sqrt{2}} \left\| \mathcal{R}_B^f(a_i) - \mathcal{R}_B^f(a_j) \right\|_2 - \frac{4}{\sqrt{2}} \sqrt{ \frac{3 ||B||_{1,2} L}{\mu} \bar{\epsilon}} \\
& \geq 
\beta([A, B]) - 2 \sqrt{ \frac{6 ||B||_{1,2} L}{\mu} \bar{\epsilon}}. 
\end{align*}
\end{proof}

 \small 

\bibliographystyle{spmpsci}  
\bibliography{Biography}

\end{document}